\documentclass{article}

\usepackage{amsthm,amsmath,amsfonts,graphpap,amscd,mathrsfs,lscape}
\usepackage{epsfig,amssymb,amstext,xspace}
\usepackage{float}	
\usepackage[numbers]{natbib}
\usepackage{nicefrac}
\usepackage[shortlabels]{enumitem}
\def\defeq{\triangleq}
\newcommand{\E}{\mathbb{E}}
\usepackage{nicefrac}
\newcommand{\ldot}[2]{\ensuremath{\langle #1, #2 \rangle}}

\newtheorem{theorem}{Theorem}
\newtheorem{lemma}[theorem]{Lemma}
\newtheorem{prop}[theorem]{Proposition}

\newtheorem{corollary}[theorem]{Corollary}

\newtheorem{assumption}{Assumption}
\newenvironment{remark}{\noindent\textbf{Remark}
  \hspace*{1em}}{\smallskip}%

\newtheorem{defn}{Definition}

\def\squareforqed{\hbox{\rule{2.5mm}{2.5mm}}}

\def\QED{\ifmmode\squareforqed %
  \else{\nobreak\hfil   %
    \penalty50                 %
    \hskip1em                  %
    \null                      %
    \nobreak                   %
    \hfil                      %
    \squareforqed              %
    \parfillskip=0pt           %
    \finalhyphendemerits=0     %
    \endgraf}                  %
  \fi}

\def\blksquare{\rule{2mm}{2mm}}
\def\qedsymbol{\blksquare}
\newcommand{\bg}[1]{\medskip\noindent{\bf #1}}
\newcommand{\ed}{{\hfill\qedsymbol}\medskip}

\newcommand{\appref}[1]{Appendix~{\ref{sec:#1}}}
\newcommand{\assumpref}[1]{Assumption~{\ref{ass:main}.\ref{ass:#1}}}
\newcommand{\assumpsref}[1]{Assumptions~{\ref{ass:main}.\ref{ass:#1}}}
\newcommand{\assumpssref}[1]{{\ref{ass:main}.\ref{ass:#1}}}
\newcommand{\assumpargref}[2]{Assumption~{\ref{ass:#1}.\ref{ass:#2}}}

\newcommand{\corref}[1]{Corollary~{\ref{cor:#1}}}

\newcommand{\eqnref}[1]{\eqref{eqn:#1}}

\newcommand{\figref}[1]{Figure~{\ref{fig:#1}}}
\newcommand{\lemref}[1]{Lemma~{\ref{lem:#1}}}

\newcommand{\defref}[1]{Definition~{\ref{def:#1}}}
\newcommand{\secref}[1]{Section~{\ref{sec:#1}}}

\newcommand{\propref}[1]{Proposition~{\ref{prop:#1}}}

\newcommand{\thmref}[1]{Theorem~{\ref{thm:#1}}}

\newcommand{\R}{\ensuremath{\mathbb R}}

\newcommand{\N}{\ensuremath{\mathbb N}}

\newcommand{\todist}{\stackrel{d}{\to}} %
\newcommand{\toprob}{\stackrel{p}{\to}} %

\newcommand{\comment}[1]{}
 {}%

\newcommand{\junk}[1]{}

\def\balign#1\ealign{\begin{align}#1\end{align}}
\def\baligns#1\ealigns{\begin{align*}#1\end{align*}}
\def\balignat#1\ealign{\begin{alignat}#1\end{alignat}}
\def\balignats#1\ealigns{\begin{alignat*}#1\end{alignat*}}
\def\bitemize#1\eitemize{\begin{itemize}#1\end{itemize}}
\def\benumerate#1\eenumerate{\begin{enumerate}#1\end{enumerate}}

\newenvironment{talign*}
 {\let\displaystyle\textstyle\csname align*\endcsname}
 {\endalign}
\newenvironment{talign}
 {\let\displaystyle\textstyle\csname align\endcsname}
 {\endalign}

\def\balignst#1\ealignst{\begin{talign*}#1\end{talign*}}
\def\balignt#1\ealignt{\begin{talign}#1\end{talign}}

\def\Holder{H\"older\xspace}

\newcommand{\qtext}[1]{\quad\text{#1}\quad} 

\newcommand{\grad}{\nabla} %

\def\defeq{\triangleq} %

\newcommand{\inner}[2]{\langle{#1},{#2}\rangle} %
\newcommand{\norm}[1]{\|{#1}\|} %

\newcommand{\Var}{\ensuremath{\text{Var}}}
\newcommand{\Cov}{\ensuremath{\text{Cov}}}

\newcommand{\hset}[0]{\mathcal{H}}

\usepackage{color}              %
\usepackage{epsfig}

\usepackage[]{color-edits}
\addauthor{vs}{red}
\addauthor{lm}{green}
\addauthor{iz}{blue}

\def\mbb#1{\mathbb{#1}}
\def\mc#1{\mathcal{#1}}

\def\tbf#1{\textbf{#1}}

\def\staticindic#1{\mbb{I}[{#1}]} %

\newcommand{\ssest}[0]{{\hat{\theta}^{SS}}} %
\newcommand{\cfest}[0]{{\hat{\theta}^{CF}}} %

\usepackage{microtype}
\usepackage{graphicx}
\usepackage{subcaption}
\usepackage[colorlinks,breaklinks]{hyperref}

\usepackage[accepted]{icml2018}

\icmltitlerunning{Orthogonal Machine Learning: Power and Limitations}

\begin{document}

\twocolumn[
\icmltitle{Orthogonal Machine Learning: Power and Limitations}

\begin{icmlauthorlist}
\icmlauthor{Lester Mackey}{msr}
\icmlauthor{Vasilis Syrgkanis}{msr}
\icmlauthor{Ilias Zadik}{msr,mit}
\end{icmlauthorlist}

\icmlaffiliation{msr}{Microsoft Research New England, USA}
\icmlaffiliation{mit}{Operations Research Center, MIT, USA}

\icmlcorrespondingauthor{Lester Mackey}{lmackey@microsoft.com}
\icmlcorrespondingauthor{Vasilis Syrgkanis}{vasy@microsoft.com}
\icmlcorrespondingauthor{Ilias Zadik}{izadik@mit.edu}

\icmlkeywords{orthogonal machine learning, double machine learning, asymptotic normality, nuisance,
semiparametric inference, partially linear regression, high-dimensional regression, Stein's lemma,
treatment effect, causal inference}

\vskip 0.3in
]

\printAffiliationsAndNotice{}  %

\begin{abstract}
Double machine learning provides $\sqrt{n}$-consistent estimates of parameters of interest even when high-dimensional or nonparametric nuisance parameters are estimated at an $n^{-1/4}$ rate. The key is to employ \emph{Neyman-orthogonal} moment equations which are first-order insensitive to perturbations in the nuisance parameters. We show that the $n^{-1/4}$ requirement can be improved to $n^{-1/(2k+2)}$ by employing a $k$-th order notion of orthogonality that grants robustness to more complex or higher-dimensional nuisance parameters. In the partially linear regression setting, popular in causal inference, we show that we can construct second-order orthogonal moments if and only if the treatment residual is not normally distributed.  Our proof relies on Stein's lemma and may be of independent interest.  We conclude by demonstrating the robustness benefits of an explicit doubly-orthogonal estimation procedure for treatment effect.
\end{abstract}

\section{Introduction}
The increased availability of large and complex observational datasets is driving an increasing demand to conduct accurate causal inference of treatment effects in the presence of high-dimensional confounding factors. We take as our running example demand estimation from pricing and purchase data in the digital economy where many features of the world that simultaneously affect pricing decisions and demand are available in large data stores.
One often appeals to modern statistical machine learning (ML) techniques to model and fit the high-dimensional or nonparametric nuisance parameters introduced by these confounders.
However, most such techniques introduce bias into their estimates (e.g., via regularization) and hence yield invalid or inaccurate inferences concerning the parameters of interest (the treatment effects).

Several recent lines of have begun address the problem of debiasing ML estimators to perform accurate inference on a low dimensional component of model parameters. Prominent examples include Lasso debiasing \citep{ZhangZh2013,vanDeGeer2013,Javanmard2015} and post-selection inference \cite{Belloni2013,Berk2013,Tibhsirani2014}. 
The recent double / debiased ML work of \citet{Chernozhukov2016} describes a general-purpose strategy for extracting valid inferences for target parameters 
from somewhat arbitrary and relatively inaccurate estimates of nuisance parameters.

Specifically, \citet{Chernozhukov2016} analyze a two-stage process where in the first stage one estimates nuisance parameters using arbitrary statistical ML techniques on a first stage data sample and in the second stage
estimates the low dimensional parameters of interest via the generalized method of moments (GMM).
Crucially, the moments in the second stage are required to satisfy a \emph{Neyman orthogonality} condition, granting them first-order robustness to errors in the nuisance parameter estimation.
A main conclusion is that the second stage estimates are $\sqrt{n}$-consistent and asymptotically normal whenever the first stage estimates are consistently estimated at a $o(n^{-1/4})$ rate. 

To illustrate this result, let us consider the partially linear regression (PLR) model, popular in causal inference. In the PLR model we observe data triplets $Z = (T,Y,X)$, where $T \in \R$ represents a treatment or policy applied, $Y \in \R$ represents an outcome of interest, and $X \in \R^p$ is a vector of associated covariates.  These observations are related via the equations \begin{align*}
&Y=\theta_0T+f_0(X)+\epsilon, \quad  \E[\epsilon \mid X, T]=0\quad a.s.\\
&T=g_0(X)+\eta,\quad \E[\eta \mid X]=0\quad a.s.
\end{align*}
where $\eta$ and $\epsilon$ represent unobserved disturbances with distributions independent of $(\theta_0, f_0, g_0)$. The first equation features the treatment effect $\theta_0$, our object of inference. The second equation describes the relation between the treatment $T$ and the associated covariates $X$. The covariates $X$ affect the outcome $Y$ through the nuisance function $f_0$ and the treatment $T$ through the nuisance function $g_0$.
Using the Neyman-orthogonal moment of \citep[Eq. 4.55]{Chernozhukov2016}, the authors show that it suffices to estimate the nuisance $(f_0,g_0)$ at an $o(n^{-1/4})$ rate to construct a $\sqrt{n}$-consistent and asymptotically normal estimator of $\theta_0$. 

In this work, we provide a framework for achieving stronger robustness to first stage errors while maintaining second stage validity. In particular, we introduce a notion of higher-order orthogonality and show that if the moment is $k$-th order orthogonal then a first-stage estimation rate of $o(n^{-1/(2k+2)})$ suffices for $\sqrt{n}$-asymptotic normality of the second stage. 

We then provide a concrete application of our approach to the case of estimating treatment effects in the PLR model.
Interestingly, we show an impossibility result when the treatment residual follows a Gaussian distribution: no higher-order orthogonal moments with finite asymptotic variance exist, so first-order Neyman orthogonality appears to be the limit of robustness to first stage errors under Gaussian treatment residual. 
However, conversely, we also show how to construct appropriate second-order orthogonal moments whenever the treatment residual is not Gaussian. 
As a result, when the nuisance functions are linear in the high-dimensional confounders, our second-order orthogonal moments provide valid inferences whenever the number of relevant confounders is $o(\frac{n^{2/3}}{\log p})$; meanwhile the first-order orthogonality analyses of \citep{Chernozhukov2016} accommodate only $o(\frac{\sqrt{n}}{\log p})$ relevant confounders. %

We apply these techniques in the setting of demand estimation from pricing and purchase data, where highly non-Gaussian treatment residuals are standard.
In this setting, the treatment is the price of a product, and commonly, conditional on all observable covariates, the treatment follows a discrete distribution representing random discounts offered to customers over a baseline price linear in the observables.
In Figure~\ref{fig:comparison} we portray the results of a synthetic demand estimation problem with dense dependence on observables. 
Here, the standard orthogonal moment estimation has large bias, comparable to variance, while our second-order orthogonal moments lead to nearly unbiased estimation.
\begin{figure*}[htpb]
\centering
\begin{subfigure}[t]{0.45\textwidth}
\centering
\includegraphics[scale=.32]{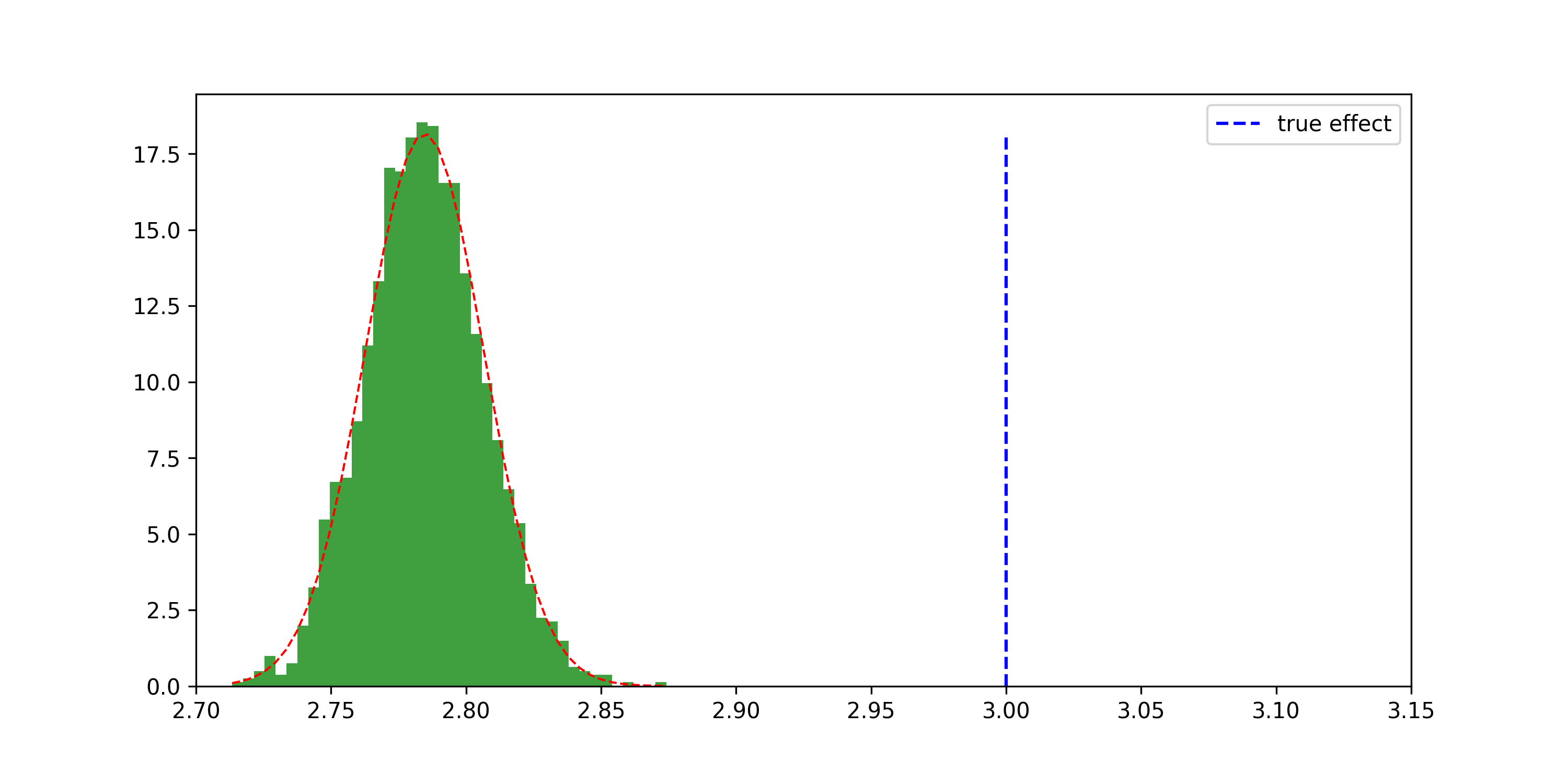}
\caption{Orthogonal estimates ($\hat{\theta}=2.78$, $\hat{\sigma}=.022$)}
\end{subfigure}
~~~~
\begin{subfigure}[t]{0.45\textwidth}
\centering
\includegraphics[scale=.32]{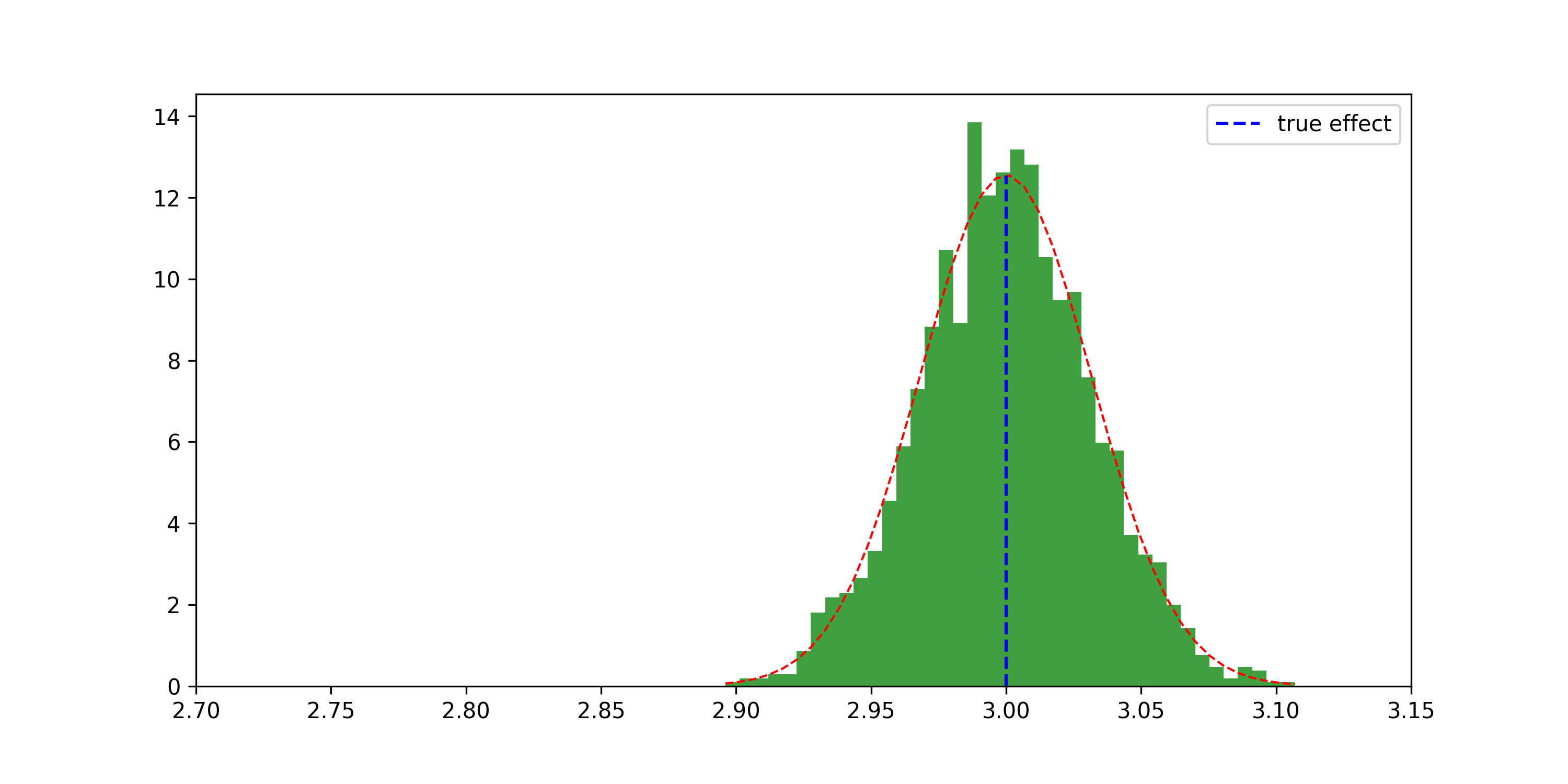}
\caption{Second-order orthogonal estimates ($\hat{\theta}=3.$, $\hat{\sigma}=.032$)}
\end{subfigure}
\caption{\small{We portray the distribution of estimates based on orthogonal moments and second-order orthogonal moments. The true treatment effect $\theta_0 = 3$. Sample size $n=5000$, dimension of confounders $d=1000$, support size of sparse linear nuisance functions $s=100$. The details of this experiment can be found in Section \ref{sec:exper}.}}
\label{fig:comparison}
\end{figure*}

\paragraph{Notational conventions}
For each $n\in \mathbb{N}$, we introduce the shorthand $[n]$ for $\{1,\dots,n\}$.
We let $\toprob$ and $\todist$ represent convergence in probability and convergence in distribution respectively.
When random variables $A$ and $B$ are independent, we use $\E_A[g(A,B)] \defeq \E[g(A,B) \mid B]$ to represent expectation only over the variable $A$.  
For a sequence of random vectors $(X_n)_{n=1}^\infty$ and a deterministic sequence of scalars $(a_n)_{n=1}^\infty$, we write $X_n = O_P(a_n)$ to mean $X_n/a_n$ is stochastically bounded, i.e., for any $\epsilon > 0$ there is $R_{\epsilon},N_{\epsilon}>0$ with $\Pr(\norm{X_n/a_n} > R_{\epsilon}) \leq \epsilon$ for all $n >N_{\epsilon}$. We let $N(\mu, \Sigma)$ represent a multivariate Gaussian distribution with mean $\mu$ and covariance $\Sigma$.

\section{$Z$-Estimation with Nuisance Functions and Orthogonality}\label{z_estimation}

Our aim is to estimate an unknown target parameter $\theta_0\in \Theta \subseteq \R^d$ given access to independent replicates $(Z_t)_{t=1}^{2n}$ of a random data vector $Z\in \R^{\rho}$ drawn from a distribution satisfying $d$ moment conditions,
\begin{equation}\label{eqn:population-moments}
\E[m(Z,\theta_0, h_0(X))|X] = 0\quad a.s.
\end{equation}
Here, $X\in \R^{p}$ is a sub-vector of the observed data vector $Z$, $h_0 \in \hset \subseteq \{h:\R^{p}\rightarrow \R^\ell\}$ is a vector of $\ell$ unknown nuisance functions, and $m: \R^{\rho} \times \R^d\times \R^\ell\rightarrow \R^d$ is a vector of $d$ known moment functions. 
We assume that these moment conditions exactly identify the parameter $\theta_0$, and we allow for the data to be high-dimensional, with $\rho$ and $p$ potentially growing with the sample size $n$. However, the number of parameters of interest $d$ and the number of nuisance functions $\ell$ are assumed to be constant.

We will analyze a two-stage estimation process where we first estimate the nuisance parameters using half of our sample\footnote{Unequal divisions of the sample can also be used; we focus on an equal division for simplicity of presentation.} and then form a $Z$-estimate of the target parameter $\theta_0$ using the remainder of the sample and our first-stage estimates of the nuisance. This \emph{sample-splitting} procedure proceeds as follows.
\begin{enumerate}
\item \emph{First stage.} Form an estimate $\hat{h} \in \hset$ of $h_0$ using $(Z_t)_{t=n+1}^{2n}$   (e.g., by running a nonparametric or high-dimensional regression procedure). 
\item \emph{Second stage.} Compute a $Z$-estimate $\ssest \in \Theta$ of $\theta_0$ using an empirical version of the moment conditions \eqnref{population-moments} and $\hat{h}$ as a plug-in estimate of $h_0$:
\begin{equation}\label{eqn:second-stage}
\textstyle
\ssest \qtext{ solves }\frac{1}{n} \sum_{t=1}^n m(Z_t, \theta, \hat{h}(X_t)) =0.
\end{equation}
\end{enumerate} 

Relegating only half of the sample to each stage represents a statistically inefficient use of data and, in many applications, detrimentally impacts the quality of the first-stage estimate $\hat{h}$. 
A form of repeated sample splitting called \emph{$K$-fold cross-fitting} \citep[see, e.g.,][]{Chernozhukov2016} addresses both of these concerns. $K$-fold cross-fitting partitions the index set of the datapoints $[2n]$ into $K$ subsets $I_1,\ldots,I_K$ of cardinality $\frac{2n}{K}$ (assuming for simplicity that $K$ divides $2n$)
and produces the following two-stage estimate:
\begin{enumerate}
\item \emph{First stage.} For each $k \in [K]$, form an estimate $\hat{h}_k \in \mathcal{H} $ of $h_0$ using only the datapoints $(Z_t)_{t \in I_k^c}$ corresponding to $I_k^c=[2n]\setminus I_k$.

\item \emph{Second stage.}  Compute a $Z$-estimate $\ssest \in \Theta$ of $\theta_0$ using an empirical version of the moment conditions and $(\hat{h}_k)_{k \in [K]}$ as plug-in estimators of $h_0$: 
\begin{equation}\label{eqn:cross}
\cfest \text{ solves }\ \frac{1}{2n} \sum_{k=1}^K  \sum_{t \in I_k} m(Z_t,\theta,\hat{h}_k(X_t))=0.
\end{equation}
\end{enumerate}
Throughout, we assume $K$ is a constant independent of all problem dimensions.
As we will see in \thmref{orth}, a chief advantage of cross-fitting over sample splitting is improved relative efficiency with an asymptotic variance that reflects the use of the full dataset in estimating $\theta$.

\emph{Main Question.} Our primary inferential goal is to establish conditions under which the estimators $\ssest$ in \eqnref{second-stage} and $\cfest$ \eqnref{cross} enjoy $\sqrt{n}$-asymptotic normality, that is
\begin{equation*}
\sqrt{n} (\ssest-\theta_0) \todist N(0,\Sigma)
\text{ and }
\sqrt{2n} (\cfest-\theta_0) \todist N(0,\Sigma)
\end{equation*}
for some constant covariance matrix $\Sigma$.
Coupled with a consistent estimator of $\Sigma$, asymptotic normality enables the construction of asymptotically valid confidence intervals for $\theta$ based on Gaussian or Student's t quantiles and asymptotically valid hypothesis tests, like the Wald test, based on chi-squared limits.

\subsection{Higher-order Orthogonality}
We would like our two-stage procedures to produce accurate estimates of $\theta_0$ even when the first stage nuisance estimates are relatively inaccurate.
With this goal in mind, \citet{Chernozhukov2016} defined the notion of Neyman-orthogonal moments, inspired by the early work of \citet{Neyman1979}. 
In our setting, the orthogonality condition of \cite{Chernozhukov2016} is implied by the following condition, which we will call \emph{first-order orthogonality}:
\begin{defn}[First-order Orthogonal Moments] A vector of moments $m: \R^\rho\times \R^d\times \R^\ell\rightarrow \R^d$ is \emph{first-order orthogonal} with respect to the nuisance $h_0(X)$ if
\begin{equation*}
\E\left[\nabla_{\gamma} m(Z,\theta_0,\gamma)|_{\gamma=h_0(X)}\, |\, X\right] = 0.
\end{equation*}
Here, $\nabla_{\gamma} m(Z,\theta_0, \gamma)$ is the gradient of the vector of moments with respect to its final $\ell$ arguments.
\end{defn}

Intuitively, first-order orthogonal moments are insensitive to small perturbations in the nuisance parameters and hence robust to small errors in estimates of these parameters.
A main result of \cite{Chernozhukov2016} is that, if the moments $m$ are first-order orthogonal, then $o(n^{-1/4})$ error rates\footnote{In the sense of root mean squared error: $n^{1/4} \sqrt{\E[\norm{h_0(X)-\hat{h}(X)}_2^2\mid \hat{h}]}\toprob 0$.} in the first stage estimation of $h_0$ are sufficient for $\sqrt{n}$-asymptotic normality of the estimates $\ssest$ and $\cfest$.

Our aim is to accommodate slower rates of convergence in the first stage of estimation by designing moments robust to larger nuisance estimation errors.
To achieve this, we will introduce a generalized notion of orthogonality that requires higher-order nuisance derivatives of $m$ to be conditionally mean zero. 
We will make use of the following higher-order differential notation: 
\begin{defn}[Higher-order Differentials] Given a vector of moments $m: \R^\rho \times \R^d\times \R^\ell\rightarrow \R^d$ and a vector $\alpha\in\N^\ell$ we denote by $D^{\alpha}m(Z,\theta, \gamma)$ the $\alpha$-differential of $m$ with respect to its final $\ell$ arguments:
\begin{equation}
D^{\alpha}m(Z,\theta, \gamma) = \nabla_{\gamma_1}^{\alpha_1} \nabla_{\gamma_2}^{\alpha_2}\ldots\nabla_{\gamma_\ell}^{\alpha_\ell}m(Z,\theta, \gamma)
\end{equation}
\end{defn}
We are now equipped to define our notion of \emph{$S$-orthogonal moments}:
\begin{defn}[$S$-Orthogonal Moments] \label{def:s-orthogonal}
A vector of moments $m: \R^\rho\times \R^d\times \R^\ell\rightarrow \R^d$ is \emph{$S$-orthogonal} with respect to the nuisance $h_0(X)$ for some \emph{orthogonality set} $S \subseteq \N^{\ell}$, if for any $\alpha \in S$: 
\begin{equation}
\E\left[D^{\alpha} m(Z,\theta_0, h_0(X))\right|X] = 0.
\end{equation}
\end{defn}
We will often be interested in the special case of \defref{s-orthogonal} in which $S$ is comprised of all vectors $\alpha \in \N^{\ell}$ with $\|\alpha\|_1\leq k$. This implies that all mixed nuisance derivatives of the moment of order $k$ or less are conditionally mean zero. 
We will refer to this special case as \emph{$k$-orthogonality} or \emph{$k$-th order orthogonality}.

\begin{defn}[$k$-Orthogonal Moments]
\label{def:k-orthogonal}
A vector of moments $m: \R^\rho\times \R^d\times \R^\ell\rightarrow \R^d$ is \emph{$k$-orthogonal} if it is $S_k$-orthogonal for 
the \emph{$k$-orthogonality set}, $S_k\defeq\{\alpha \in \N^{\ell}: \|\alpha\|_1\leq k\}$.
\end{defn}
The general notion of $S$-orthogonality allows for our moments to be more robust to errors in some nuisance functions and less robust to errors in others. 
This is particularly valuable when some nuisance functions are easier to estimate than others; we will encounter such an example in \secref{plr-power}.

\section{Higher-order Orthogonality and Root-$n$ Consistency}\label{sec:root_n_consistency}

We will now show that $S$-orthogonality together with appropriate consistency rates for the first stage estimates of the nuisance functions imply $\sqrt{n}$-consistency and asymptotic normality of the two-stage estimates $\ssest$ and $\cfest$. 
Beyond orthogonality and consistency, our main Assumption~\ref{ass:main} demands identifiability, non-degeneracy, and regularity of the moments $m$, all of which are standard for establishing the asymptotic normality of $Z$-estimators. 

\begin{assumption}\label{ass:main} For a non-empty orthogonality set $S \subseteq \N^{\ell}$ and $k\defeq\max_{\alpha \in S} \|\alpha\|_1$, we assume the following:
\begin{enumerate}
\item\label{ass:orthogonality} \tbf{\emph{$S$-Orthogonality.}} The moments $m$ are $S$-orthogonal.
\item\label{ass:identifiable}  \tbf{\emph{Identifiability.}} $\E[m(Z,\theta, h_0(X))] \neq 0$ when $\theta \neq \theta_0$. %
\item\label{ass:reg_full_rank} 
	\tbf{\emph{Non-degeneracy.}} The matrix $\E\left[\nabla_\theta m(Z,\theta_0,h_0(X))\right]$ is invertible.
\item\label{ass:continuity} 
	\tbf{\emph{Smoothness.}} $\grad^k m$ exists and is continuous.
\item\label{ass:first-stage-consistency} \tbf{\emph{Consistency of First Stage.}} The first stage estimates satisfy 
\begin{equation*}
\textstyle
\E[\prod_{i=1}^\ell |\hat{h}_i(X)-h_{0,i}(X)|^{4\alpha_i} \mid \hat{h} ] \toprob 0,
\quad\forall \alpha \in S,
\end{equation*}
where the convergence in probability is with respect to the auxiliary data set used to fit $\hat{h}$.
\item\label{ass:first-stage} \tbf{\emph{Rate of First Stage.}} The first stage estimates satisfy
\begin{equation*}
\textstyle
n^{1/2}\cdot\sqrt{\E[\prod_{i=1}^\ell |\hat{h}_i(X)-h_{0,i}(X)|^{2\alpha_i} \mid \hat{h} ]} \toprob 0,
\end{equation*}
$ \forall \alpha \in \{ a \in \mathbb{N}^{\ell}: \|a\|_1 \leq k+1 \} \setminus S$, where the convergence in probability is with respect to the auxiliary data set used to fit $\hat{h}$.
\item\label{ass:reg_moments} \tbf{\emph{Regularity of Moments.}} There exists an $r>0$ such that the following regularity conditions hold:
\begin{enumerate}
\item\label{ass:theta_dominated} 
 
	$\E[\sup_{\theta\in\mc{B}_{\theta_0,r}} \norm{\grad_\theta m(Z, \theta, h_0(X))}_F] < \infty$\\ 
for $\mc{B}_{\theta_0,r} \defeq \{ \theta \in \Theta: \|\theta-\theta_0\|_2 \leq r\}.$ %
\item\label{ass:gamma_theta_derivs} \mbox{}\\
\vspace{-.125\linewidth}
$$\displaystyle\sup_{h\in\mc{B}_{h_0,r}} \E[\sup_{\theta\in\mc{B}_{\theta_0,r}} \norm{\grad_\gamma \grad_\theta m(Z, \theta,h(X))}^2] < \infty$$
 for $\mc{B}_{h_0,r} \defeq
 \{ h \in \hset : $\\ 
 $\displaystyle\max_{\alpha:\norm{\alpha}_1 \leq k+1}\textstyle\E[\prod_{i=1}^\ell |h_i(X)-h_{0,i}(X)|^{2\alpha_i} ] \leq r\}.$
\item\label{ass:reg_bound_diff}
	$\displaystyle\max_{\alpha:\norm{\alpha}_1 \leq k+1} \sup_{h\in\mc{B}_{h_0,r}}  \E\left[|D^{\alpha}m(Z, \theta_0, h(X))|^4\right] \leq \lambda_*(\theta_0, h_0) < \infty$.
\item  $\E[\sup_{\theta\in A, h \in \mc{B}_{h_0,r}} \norm{m(Z,\theta, h(X))}_2] < \infty$,\\  for any compact $A\subseteq \Theta$,
 \label{ass:reg_m_dominated} %
\item $\sup_{\theta\in A, h \in \mc{B}_{h_0,r}} \E[\norm{\grad_\gamma m(Z,\theta, h(X)) }^2] < \infty$,\\ for any compact $A\subseteq \Theta$.
\label{ass:reg_grad_gamma} %
\end{enumerate}
\end{enumerate}
\end{assumption}

We are now ready to state our main theorem on the implications of $S$-orthogonality for second stage $\sqrt{n}$-asymptotic normality.
The proof can be found in \secref{orth-proof}.
\begin{theorem}[Main Theorem]\label{thm:orth} Under Assumption \ref{ass:main}, if $\ssest$ and $\cfest$ are consistent, then 
\begin{equation*}
\sqrt{n} (\ssest-\theta_0) \todist N(0,\Sigma)
\text{ and }
\sqrt{2n} (\cfest-\theta_0) \todist N(0,\Sigma)
\end{equation*}
where $\Sigma = J^{-1} V J^{-1}$ for $J = \E\left[\nabla_\theta m(Z,\theta_0, h_0(X))\right]$ and $V = \mathtt{Cov}(m(Z,\theta_0, h_0(X)))$. 
\end{theorem}

A variety of standard sufficient conditions guarantee the consistency of $\ssest$ and $\cfest$.  Our next result, proved in \secref{proof-consistency}, establishes consistency under either of two commonly satisfied assumptions.
\begin{assumption}\label{ass:consistency}
One of the following sets of conditions is satisfied:
\begin{enumerate}
\item\tbf{\emph{Compactness conditions:}}\label{ass:compactness} $\Theta$ is compact. 
\item\tbf{\emph{Convexity conditions:}}\label{ass:convexity} 
$\Theta$ is convex, $\theta_0$ is in the interior of $\Theta$, and, with probability approaching 1,
the mapping $\theta \mapsto \frac{1}{n} \sum_{t=1}^n m(Z_t, \theta, \hat{h}(X_t))$ is the gradient of a convex function.
\end{enumerate}
\end{assumption}
\begin{remark}
A continuously differentiable vector-valued function $\theta \mapsto F(\theta)$
on a convex domain $\Theta$ 
is the gradient of a convex function whenever the matrix $\nabla_\theta F(\theta)$ is symmetric
and 
positive semidefinite for all $\theta$.
\end{remark}
\begin{theorem}[Consistency]\label{thm:consistency}
If Assumptions \ref{ass:main} and \ref{ass:consistency} hold, then $\ssest$ and $\cfest$ are consistent.
\end{theorem}

\subsection{Sufficient Conditions for First Stage Rates}
Our assumption on the first stage estimation rates, i.e., that $\forall \alpha \in \{ a \in \mathbb{N}^{\ell}: \|a\|_1 \leq k+1 \} \setminus S$
\balignst
n^{1/2}\cdot \sqrt{\E\left[\prod_{i=1}^\ell |\hat{h}_i(X)-h_{0,i}(X)|^{2\alpha_i} \mid \hat{h} \right]} \toprob 0
\ealignst
may seem complex, as it involves the interaction of the errors of multiple nuisance function estimates. In this section we give sufficient conditions that involve only the rates of individual nuisance function estimates and which imply our first stage rate assumptions. In particular, we are interested in formulating consistency rate conditions for each nuisance function $h_i$ with respect to an $\mathcal{L}^p$ norm, 
\begin{equation*}
\|\hat{h}_i - h_{0,i}\|_{p} = \E[ \|\hat{h}_i(X) - h_{0,i}(X)\|_p^{p} \mid\hat{h}]^{1/p}.
\end{equation*}
We will make use of these sufficient conditions when applying our main theorem to the partially linear regression model in \secref{plr-power}.

\begin{lemma}\label{thm:Sorth}
Let $k = \max_{a\in S}\|a\|_1$.  Then
\begin{itemize}
\item[(1)] \assumpref{first-stage} holds if any of the following holds $\forall \alpha \in \{ a \in \mathbb{N}^{\ell}: \|a\|_1 \leq k+1 \} \setminus S$:
\begin{align}
\label{eqn:suff-cond-1}
&\bullet\quad\textstyle\sqrt{n} \prod_{i=1}^\ell \|\hat{h}_i - h_{0,i}\|_{2\|\alpha\|_1}^{\alpha_i} \toprob 0 \\
\label{eqn:suff-cond-2}
&\bullet\quad \forall i,\quad n^{\frac{1}{\kappa_i \|\alpha\|_1}}  \|\hat{h}_i - h_{0,i}\|_{2\|\alpha\|_1} \toprob 0\\
&  \qquad\text{for some }  \kappa_i \in (0,2]
\text{ where }
\textstyle\frac{1}{\|\alpha\|_1}\sum_{i=1}^\ell \frac{\alpha_i}{\kappa_i} \geq  \frac{1}{2}\notag\\
\label{eqn:suff-cond-3}
&\bullet\quad \forall i,\quad n^{\frac{1}{\kappa_i \|\alpha\|_1}}  \|\hat{h}_i - h_{0,i}\|_{2\|\alpha\|_1} \toprob 0\\
&\qquad\text{for some}\quad
\kappa_i \in (0,2].  \notag
\end{align}
\item[(2)] \assumpref{first-stage-consistency} holds if $\forall i$, $\|\hat{h}_i - h_{0,i}\|_{4k} \toprob 0$.
\end{itemize}
\end{lemma}

A simpler description of the sufficient conditions arises under $k$-orthogonality (\defref{k-orthogonal}), since the set
$\{ a \in \mathbb{N}^{\ell}: \|a\|_1 \leq k+1 \} \setminus S_k$ contains only vectors $\alpha$ with $\|\alpha\|=k+1$.
\begin{corollary}\label{cor:korth}
If $S$ is the canonical $k$-orthogonality set $S_k$ (\defref{k-orthogonal}), then \assumpref{first-stage} holds whenever 
\begin{equation*}
\forall i,\quad n^{\frac{1}{2(k+1)}} \|\hat{h}_i - h_{0,i}\|_{2(k+1)} \toprob 0,
\end{equation*} 
and \assumpref{first-stage-consistency} holds whenever $\forall i$, $\|\hat{h}_i - h_{0,i}\|_{4k} \toprob 0$.
\end{corollary}

In the case of first-order orthogonality, \corref{korth} requires that the first stage nuisance functions be estimated at a $o(n^{-1/4})$ rate with respect to the $\mathcal{L}^{4}$ norm. This is almost but not exactly the same as the condition presented in \cite{Chernozhukov2016}, which require $o(n^{-1/4})$ consistency rates with respect to the $\mathcal{L}^{2}$ norm. Ignoring the expectation over $X$, the two conditions are equivalent.\footnote{We would recover the exact condition in \cite{Chernozhukov2016} if we replaced \assumpref{reg_bound_diff} with the more stringent assumption that $\left|D^{\alpha}m(Z, \theta, h(X))\right| \leq \lambda_*$ a.s.} Moreover, in the case of $k$-orthogonality, \corref{korth} requires $o(n^{-1/2(k+1)})$ rates with respect to the $\mathcal{L}^{2(k+1)}$ norm. More generally, $S$-orthogonality allows for some functions to be estimated slower than others as we will see in the case of the sparse linear model.

\section{Second-order Orthogonality for Partially Linear Regression}\label{sec:partial_linear_model}

When second-order orthogonal moments satisfying Assumption~\ref{ass:main} are employed, \corref{korth} implies that an $o(n^{-1/6})$ rate
of nuisance parameter estimation is sufficient for $\sqrt{n}$-consistency of $\ssest$ and $\cfest$.
This asymptotic improvement over first-order orthogonality holds the promise of accommodating more complex and higher-dimensional nuisance parameters.
In this section, we detail both the limitations and the power of this approach in the partially linear regression (PLR) model setting popular in causal inference 
\citep[see, e.g,][]{Chernozhukov2016}. %

\begin{defn}[Partially Linear Regression (PLR)]~\label{def:plr}
In the partially linear regression model of observations $Z = (T,Y,X)$, 
$T \in \R$ represents a treatment or policy applied,
$Y \in \R$ represents an outcome of interest, 
and $X \in \R^p$ is a vector of associated covariates.  
These observations are related via the equations
\begin{align*}
&Y=\theta_0T+f_0(X)+\epsilon, \quad  \E[\epsilon \mid X, T]=0\quad a.s.\\
&T=g_0(X)+\eta,\quad \E[\eta \mid X]=0\quad a.s.
\end{align*}
where $\eta$ and $\epsilon$ represent unobserved noise variables with distributions independent of $(\theta_0, f_0, g_0)$.
\end{defn}

\subsection{Limitations: the Gaussian Treatment Barrier}
Our first result shows that, under the PLR model, if the treatment noise, $\eta$, is conditionally Gaussian given $X$, then no second-order orthogonal moment can satisfy Assumption~\ref{ass:main}, because every twice continuously differentiable $2$-orthogonal moment
has $\E\left[\nabla_\theta m(Z, \theta_0, h_0(X))\right] = 0$ (a violation of Assumption~\ref{ass:main}.3).
The proof in \secref{limitations-proof} relies on Stein's lemma.

\begin{theorem}\label{thm:limitations}
Under the PLR model, suppose that $\eta$ is conditionally Gaussian given $X$ (a.s.\ X).
If a twice differentiable moment function $m$ is second-order orthogonal with respect to the nuisance parameters $(f_0(X),g_0(X))$, then it must satisfy  $\E\left[\nabla_\theta m(Z, \theta_0, h_0(X))\right] = 0$ and hence violate Assumption~\ref{ass:main}.3.
Therefore no second-order orthogonal moment satisfies Assumption~\ref{ass:main}. 
\end{theorem}

In the following result, proved in \secref{lemma_degeneracy}, we establish that under mild conditions \assumpref{reg_full_rank} is necessary for the $\sqrt{n}$-consistency of $\ssest$ in the PLR model.

\begin{prop}\label{prop:deg}
Under the PLR model, suppose that $|\Theta| \geq 2$ and that the conditional distribution of $(\epsilon, \eta)$ given $X$ has full support on $\R^2$ (a.s.\ X).
Then no moment function $m$ simultaneously satisfies
\begin{enumerate}
\item Assumption~\ref{ass:main}, except for \assumpref{reg_full_rank},
\item $\E\left[\nabla_\theta m(Z, \theta_0, h_0(X))\right] = 0$, and
\item $\ssest-\theta_0 = O_P(1/\sqrt{n})$.
\end{enumerate} 
\end{prop} 

\subsection{Power: Second-order Orthogonality under Non-Gaussian Treatment}
\label{sec:plr-power}
We next show that, inversely, second-order orthogonal moments are available whenever the conditional distribution of treatment noise given $X$ is not a.s.\ Gaussian.
Our proofs rely on a standard characterization of a Gaussian distribution, proved in \secref{proof-MGFgaussian}:
\begin{lemma}\label{MGFgaussian}
If $\E[\eta|X]=0$ a.s., the conditional distribution of $\eta$ given $X$ is a.s.\ Gaussian if and only if for all $r \in \mathbb{N},r \geq 2$ it holds that, $\E \left[ \eta^{r+1} |X\right] = r \E[\eta^2|X]\E \left[ \eta^{r-1} |X\right] $ a.s.
\end{lemma}
We will focus on estimating the nuisance functions $q_0=f_0+\theta_0g_0$ and $g_0$ instead of the nuisance functions $f_0$ and $g_0$,
since the former task is more practical in many applications.
This is because estimating $q_0$ can be accomplished by carrying out an arbitrary non-parametric regression of $Y$ onto $X$. In contrast, estimating $f_0$ typically involves regressing $Y$ onto $(X,T)$, where $T$ is constrained to enter linearly. The latter might be cumbersome when using arbitrary ML regression procedures.

Our first result, established in \secref{proof-h1}, produces finite-variance 2-orthogonal moments when an appropriate moment of the treatment noise $\eta$ is known.
\begin{theorem}\label{thm:h1}
Under the PLR model, suppose that we know $\E[\eta^r|X]$ and that $\E[\eta^{r+1}]\not = r\E[\E[\eta^2|X]\E[\eta^{r-1}|X]]$ for some $r\in\N$, 
so that the conditional distribution of $\eta$ given $X$ is \textbf{not} a.s.\ Gaussian. 
Then the moments
\begin{align*}
m&\left(Z, \theta, q(X), g(X),\mu_{r-1}(X)\right) \\
\defeq & \left(Y-q(X)-\theta \left(T-g(X)\right) \right)\\
&\times \left(\left(T-g(X)\right)^r-\E[\eta^r|X]- r\left(T-g(X)\right)\mu_{r-1}(X)\right)
\end{align*} 
satisfy each of the following properties
\begin{itemize}
\item \tbf{\emph{2-orthogonality}} with respect to the nuisance $h_0(X) = (q_0(X),g_0(X),\E[\eta^{r-1}|X])$,
\item \tbf{\emph{Identifiability:}} When $\theta \neq \theta_0,$ $$\E[m(Z,\theta, q_0(X),g_0(X),\E[\eta^{r-1}|X])] \neq 0,$$
\item \tbf{\emph{Non-degeneracy:}} $$\E[\nabla_{\theta}m\left(Z, \theta_0, q_0(X),g_0(X),\E[\eta^{r-1}|X]\right)] \not =0,$$
\item \tbf{\emph{Smoothness:}} $\grad^k m$ is continuous for all $k\in\N$.
\end{itemize} 
\end{theorem}

Our next result, proved in \secref{proof-unknown2}, addresses the more realistic setting in which we do not have exact knowledge of $\E \left[ \eta^{r} |X\right]$. We introduce an additional nuisance parameter and still satisfy an orthogonality condition with respect to these parameters. 

\begin{theorem}\label{thm:unknown2}
Under the PLR model, suppose that $\E[\eta^{r+1}]\not = r\E[\E[\eta^2|X]\E[\eta^{r-1}|X]]$ for $r\in\N$,
so that the conditional distribution of $\eta$ given $X$ is \textbf{not} a.s.\ Gaussian. 
Then, if 
\[
S\defeq\{\alpha \in \mathbb{N}^4 : \|\alpha\|_1 \leq 2\}\setminus \{(1,0,0,1),(0,1,0,1)\},
\] the moments
\begin{align*}
m&\left(Z, \theta, q(X),g(X),\mu_{r-1}(X), \mu_{r}(X)\right)\\
\defeq &\left(Y-q(X)-\theta \left(T-g(X)\right) \right)  \\
& \times \left(\left(T-g(X)\right)^r-\mu_{r}(X)- r\left(T-g(X)\right)\mu_{r-1}(X)\right)
\end{align*} 
satisfy each of the following properties
\begin{itemize}
\item \tbf{\emph{$S$-orthogonality}} with respect to the nuisance $h_0(X) = (q_0(X),g_0(X),\E[\eta^{r-1}|X], \E[\eta^{r}|X])$,
\item \tbf{\emph{Identifiability:}} When $\theta \neq \theta_0$, $$\E[m(Z,\theta, q_0(X),g_0(X),\E[\eta^{r-1}|X], \E[\eta^{r}|X])] \neq 0,$$ 
\item \tbf{\emph{Non-degeneracy:}} \begin{align*}&\E[\nabla_{\theta}m\left(Z, \theta_0, q_0(X),g_0(X),\E[\eta^{r-1}|X], \E[\eta^{r}|X]\right)]\\
& \not =0,\end{align*}
\item \tbf{\emph{Smoothness:}} $\grad^k m$ continuous for all $k\in\N$.
\end{itemize}

In words, $S$-orthogonality here means that $m$ satisfies the orthogonality condition for all mixed derivatives of total order at most 2 with respect to the four nuisance parameters except the mixed derivatives with respect to $(q_0(X),\E[\eta^{r}|X])$ and $(g_0(X),\E[\eta^{r}|X])$. 
\end{theorem}

\subsection{Application to High-dimensional Linear Nuisance Functions}\label{sec:sparse_linear_model}
We now consider deploying the PLR model in the high-dimensional linear regression setting, where
$f_0(X) = \inner{X}{\beta_0}$ and $g_0(X) = \inner{X}{\gamma_0}$ for two $s$-sparse vectors $\beta_0, \gamma_0 \in \mathbb{R}^p$, 
$p$ tends to infinity as $n\to \infty$,
and $(\eta, \epsilon, X)$ are mutually independent.
Define $q_0 = \theta_0 \beta_0 + \gamma_0$.
In this high-dimensional regression setting, \citet[Rem. 4.3]{Chernozhukov2016} showed that two-stage estimation with first-order orthogonal moments 
\balignt\label{eqn:first-order-moments}
&m\left(Z, \theta, \inner{X}{q},\inner{X}{\gamma}\right)=\\ \notag
&\left(Y-\inner{X}{q}-\theta \left(T-\inner{X}{\gamma}\right) \right)\left(T-\inner{X}{\gamma}\right)
\ealignt
and Lasso estimates of the nuisance provides a $\sqrt{n}$-asymptotically normal estimator of $\theta_0$ when $s = o({n^{\frac{1}{2}}}{/\log p})$. 
Our next result, established in \appref{proof-sparsethm}, shows that we can accommodate  $s = o({n^{\frac{2}{3}}}{/\log p})$ with an explicit set of higher-order orthogonal moments.
\begin{theorem}\label{thm:sparsethm} 
In the high-dimensional linear regression setting, suppose that either $\mathbb{E}[\eta^3] \not = 0$ (non-zero skewness) or $\mathbb{E}[\eta^4] \not = 3\mathbb{E}[\eta^2]^2$ (excess kurtosis),
that $X$ has i.i.d.\ mean-zero standard Gaussian entries,
that $\epsilon$ and $\eta$ are almost surely bounded by the known value $C$,
and that $\theta_0 \in [-M,M]$ for known $M$.
If 
$
s=o({n^{2/3}/}{\log p}),
$
and in the first stage of estimation we
\begin{itemize}
\item[(a)] create estimates $\hat{q},\hat{\gamma}$ of $q_0,\gamma_0$ via Lasso regression of $Y$ on $X$ and $T$ on $X$ respectively, with regularization parameter $\lambda_n=2CM\sqrt{3\log(p)/n}$ and
\item[(b)] estimate $\mathbb{E}[\eta^2]$ and $\mathbb{E}[\eta^3]$ using $\hat{\eta}_t \defeq T_t' - \inner{X_t'}{\hat{\gamma}}$, 
\balignst
 &\hat{\mu}_{2} = \frac{1}{n}\sum_{t=1}^n\hat{\eta}_t^2,
\text{ and }
\hat{\mu}_{3} = \frac{1}{n}\sum_{t=1}^n(\hat{\eta}_t^3 
- 3\hat{\mu}_{2}\hat{\eta}_t),
\ealignst
for $(T_t', X_t')_{t=1}^n$ an i.i.d. sample independent of $\hat{\gamma}$,
\end{itemize} then, using the moments $m$ of \thmref{unknown2} with $r=2$ in the case of non-zero skewness or $r=3$ in the case of excess kurtosis, $\ssest$ and $\cfest$ are $\sqrt{n}$-asymptotically normal estimators of $\theta_0$. 
\end{theorem}

\section{Experiments}\label{sec:exper}
\label{sec:monte_carlo}
\label{sec:varying_sparsity}

We perform an experimental analysis of the second order orthogonal estimator of \thmref{sparsethm} with $r=3$ for the case of estimating treatment effects in the PLR model with high-dimensional sparse linear nuisance functions. We compare our estimator with the double ML estimator (labeled `dml' in our figures) based on the 
first-order orthogonal moments \eqnref{first-order-moments} of \citep{Chernozhukov2016}. 
Our experiments are designed to simulate demand estimation from pricing and purchase data, where non-Gaussian treatment residuals are standard.
Here, our covariates $X$ correspond to all collected variables that may affect a pricing policy. 
A typical randomized experiment in a pricing policy takes the form of random discounts from a baseline price as a company offers random discounts to customers periodically to gauge demand level. In this case, the treatment residual  -- the unexplained fluctuation in price -- is decidedly non-Gaussian and specifically follows a discrete distribution over a small number of price points.  
Python code recreating all experiments is available at \url{https://github.com/IliasZadik/double_orthogonal_ml}.

\begin{figure*}[htpb]
\centering
\includegraphics[scale=.455]{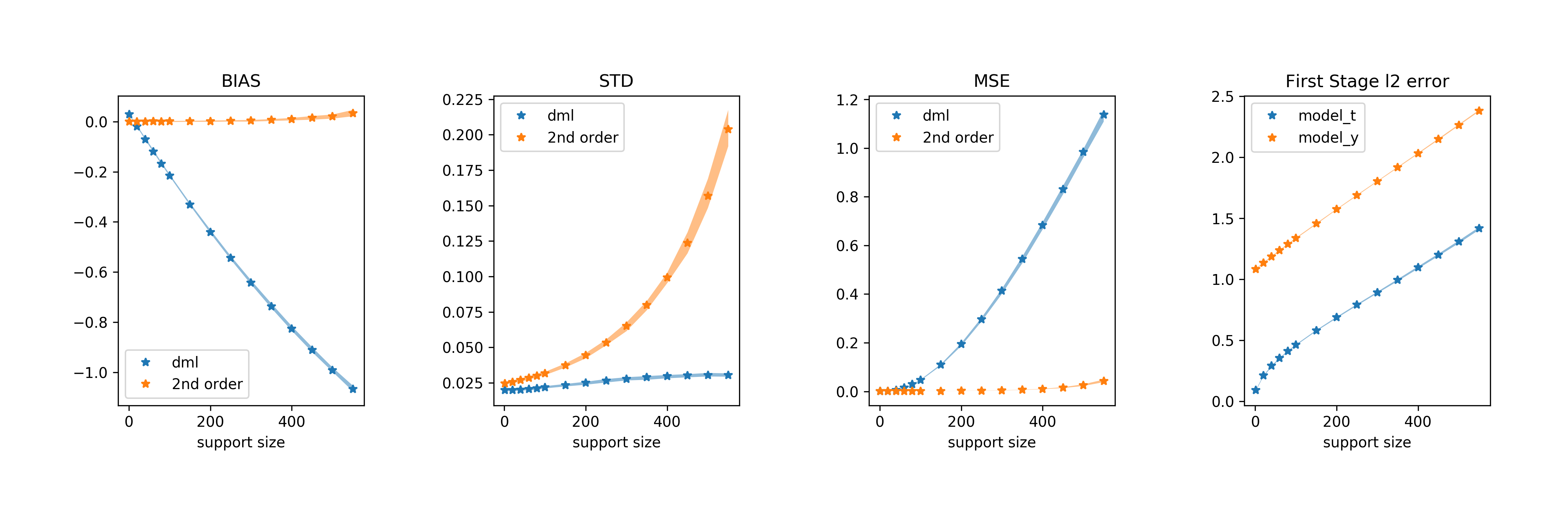}
\caption{\small{Comparison of estimates $\cfest$ based on orthogonal moments and second order orthogonal moments under the PLR model as a function of the number of non-zero coefficients in the nuisance vectors $\gamma_0$ and $\beta_0$.  See \secref{varying_sparsity} for more details. The parameters used for this figure were $n=5000$, $p=1000$, $\sigma_{\epsilon}=1$}. The fourth figure displays the $\ell_2$ error in the coefficients discovered by the first stage estimates for each of the nuisance functions: model\_t is the model for $\E[T|X]$ and model\_y is the model for $\E[Y|X]$.}
\label{fig:varying_support_multi_instance}
\end{figure*}

\paragraph{Experiment Specification} We generated $n$ independent replicates of outcome $Y$, treatment $T$, and confounding covariates $X$. The confounders $X$ have dimension $p$ and have independent components from the $N(0,1)$ distribution. The treatment is a sparse linear function of $X$, $T= \ldot{\gamma_0}{X} + \eta$, where only $s$ of the $p$ coefficients of $\gamma_0$ are non-zero. The $x$-axis on each plot is the number of non-zero coefficients $s$. Moreover, $\eta$ is drawn from a discrete distribution with values $\{0.5, 0, -1.5, -3.5\}$ taken respectively with probabilities $(.65, .2, .1, .05)$. Here, the treatment represents the price of a product or service, and this data generating process simulates random discounts over a baseline price. Finally, the outcome is generated by a linear model, $Y = \theta_0 T + \ldot{\beta_0}{X} + \epsilon$, where $\theta_0=3$ is the treatment effect, $\beta_0$ is another sparse vector with only $s$ non-zero entries, and $\epsilon$ is drawn independently from a uniform $U(-\sigma_\epsilon, \sigma_\epsilon)$ distribution. Importantly, the coordinates of the $s$ non-zero entries of the coefficient $\beta_0$ are the same as the coordinates of the $s$ non-zero entries of $\gamma_0$. The latter ensures that variables $X$ create a true endogeneity problem, i.e., that $X$ affects both the treatment and the outcome directly. In such settings, controlling for $X$ is important for unbiased estimation.

To generate an instance of the problem, the common support of both $\gamma_0$ and $\beta_0$ was generated uniformly at random from the set of all coordinates, and each non-zero coefficient was generated independently from a uniform $U(0,5)$ distribution. The first stage nuisance functions were fitted for both methods by running the Lasso on a subsample of $n/2$ sample points. 
For the first-order method all remaining $n/2$ points were used for the second stage estimation of $\theta_0$. 
For the second-order method, the moments $\mathbb{E}[\eta^2]$ and $\mathbb{E}[\eta^3]$ were estimated using a subsample of $n/4$ points as described in \thmref{sparsethm}, and the remaining $n/4$ sample points were used for the second stage estimation of $\theta_0$. For each method we performed cross-fitting across the first and second stages, and for the second-order method we performed nested cross-fitting between the $n/4$ subsample used for the $\mathbb{E}[\eta^2]$ and $\mathbb{E}[\eta^3]$ estimation and the $n/4$ subsample used for the second stage estimation. The regularization parameter $\lambda_n$ of each Lasso was chosen to be $\sqrt{\log(p)/n}$.

For each instance of the problem, i.e., each random realization of the coefficients, we generated $2000$ independent datasets to estimate the bias and standard deviation of each estimator. We repeated this process over $100$ randomly generated problem instances, each time with a different draw of the coefficients $\gamma_0$ and $\beta_0$, to evaluate variability across different realizations of the nuisance functions. 

\paragraph{Distribution of Errors with Fixed Sparsity}
In Figure \ref{fig:comparison},
we display the distribution of estimates based on orthogonal moments and second-order orthogonal moments for a particular sparsity level $s=100$ and for $n=5000$ and $p=1000$. We observe that both estimates are approximately normally distributed, but the orthogonal moment estimation exhibits significant bias, an order of magnitude larger than the variance. 

\paragraph{Bias-Variance Tradeoff with Varying Sparsity} Figure \ref{fig:varying_support_multi_instance} portrays the median quantities (solid lines) and  maximum and minimum of these quantities (error bars) across the $100$ different nuisance function draws as a function of the support size for $n=5000$, $p=1000$, and $\sigma_\epsilon=1$. 

\paragraph{Varying $n$ and $p$} In \figref{comparison-mse}, we display how performance varies with $n$ and $p$. Due to computational considerations, for this parameter exploration, we only used a single problem instance for each $(n, p, s)$ triplet rather than $100$ instances as in the exploration above. We note that for $n=2000, p=5000$ the breaking point of our method is around $s=100$, while for $n=5000, p=2000$ it is around $s=550$. For $n=10000, p=1000$ our method performs exceptionally well even until $s=800$.

\begin{figure}[htpb]
\begin{subfigure}[t]{0.15\textwidth}
\centering
\includegraphics[scale=.23]{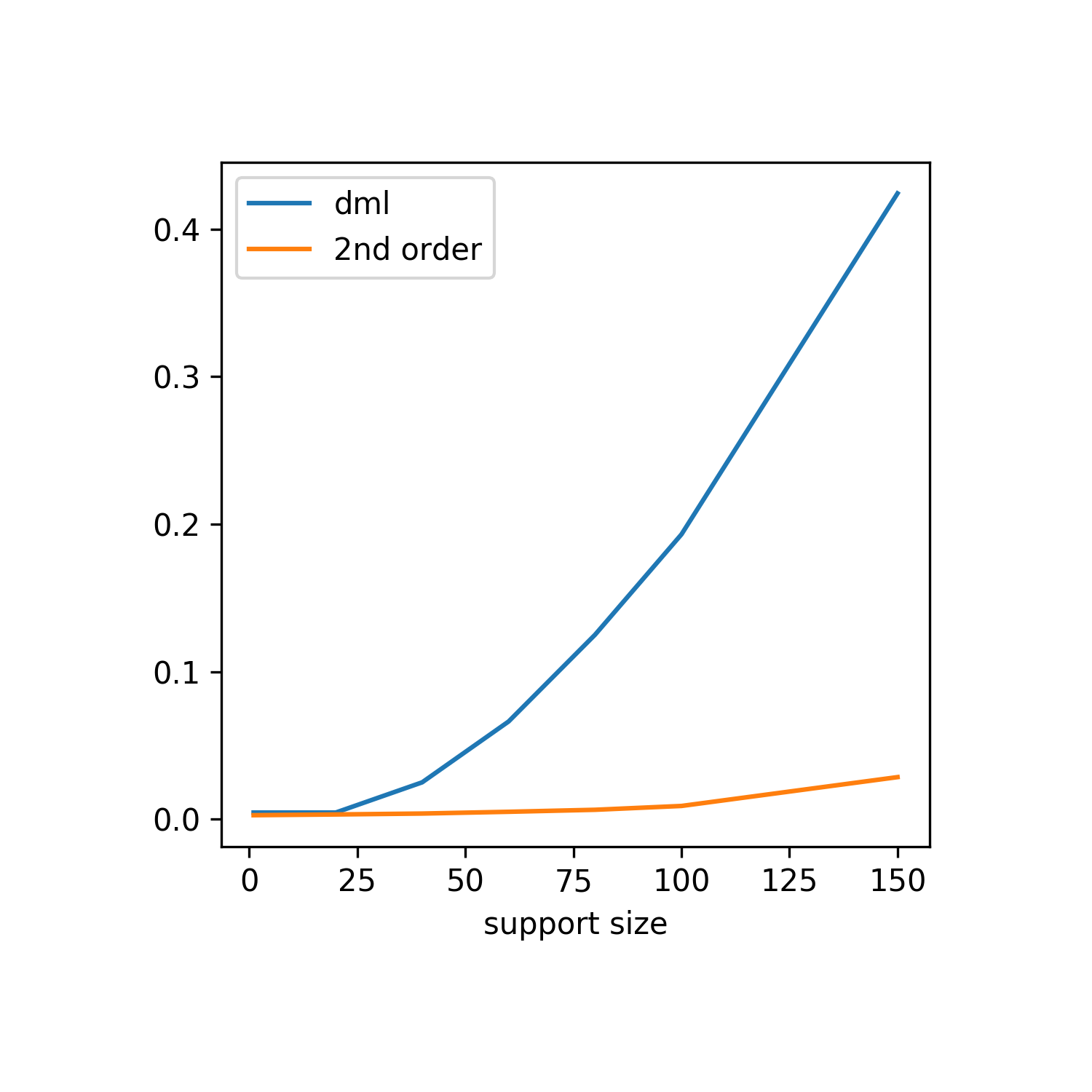}
\caption{\tiny{$n=2000, p=1000$}}
\end{subfigure}
~
\begin{subfigure}[t]{0.15\textwidth}
\centering
\includegraphics[scale=.23]{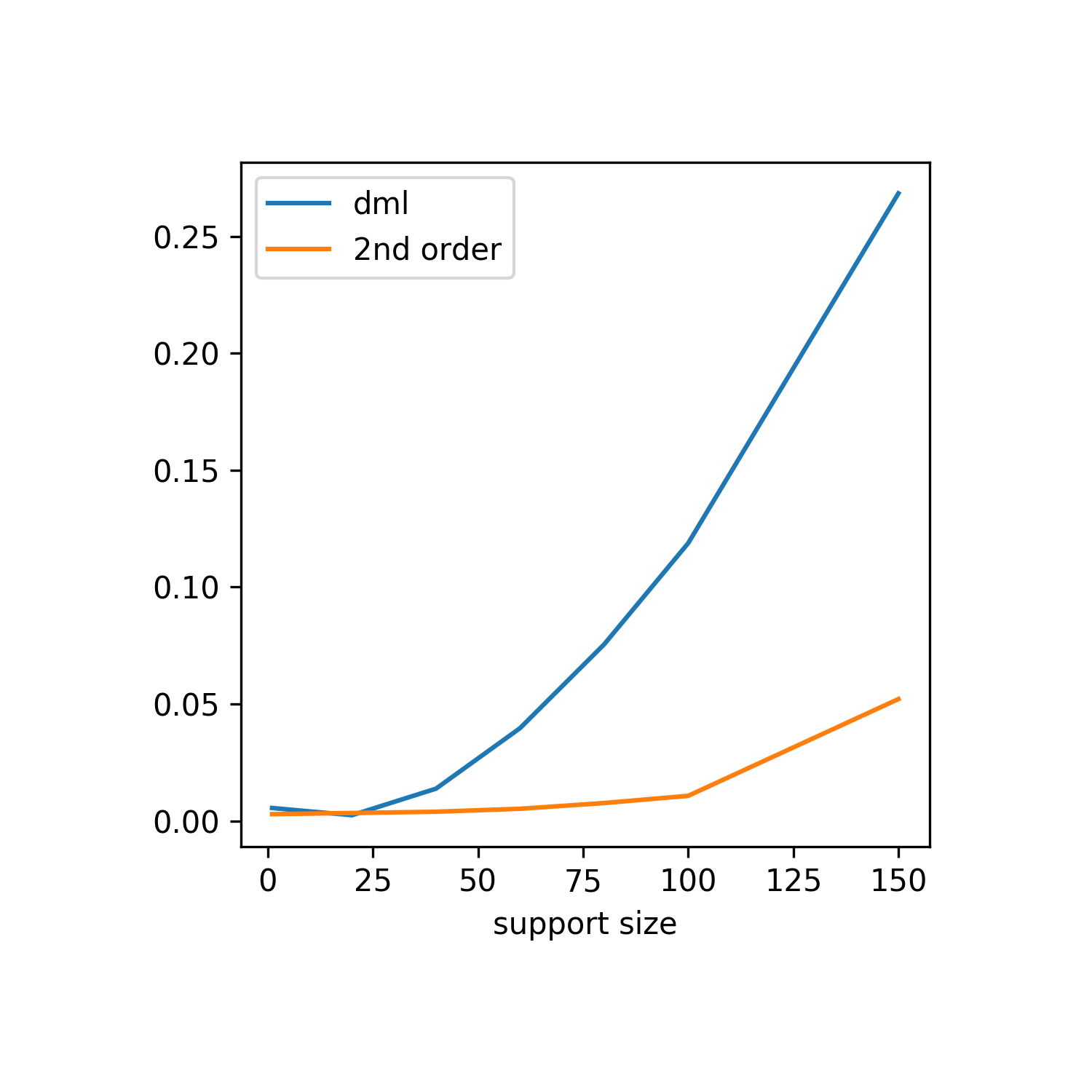}
\caption{\tiny{$n=2000, p=2000$}}
\end{subfigure}
~
\begin{subfigure}[t]{0.15\textwidth}
\centering
\includegraphics[scale=.23]{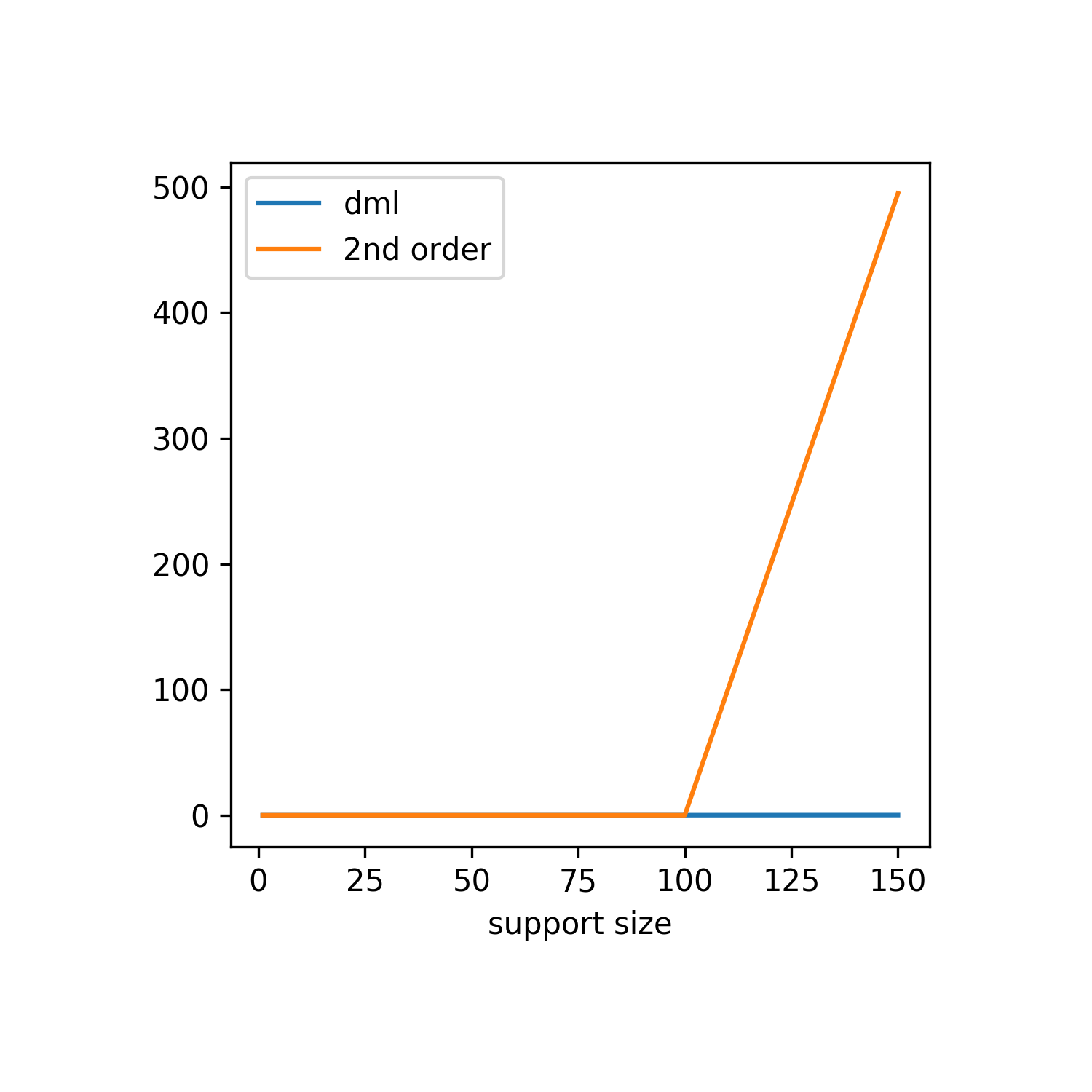}
\caption{\tiny{$n=2000, p=5000$}}
\end{subfigure}

\begin{subfigure}[t]{0.15\textwidth}
\centering
\includegraphics[scale=.23]{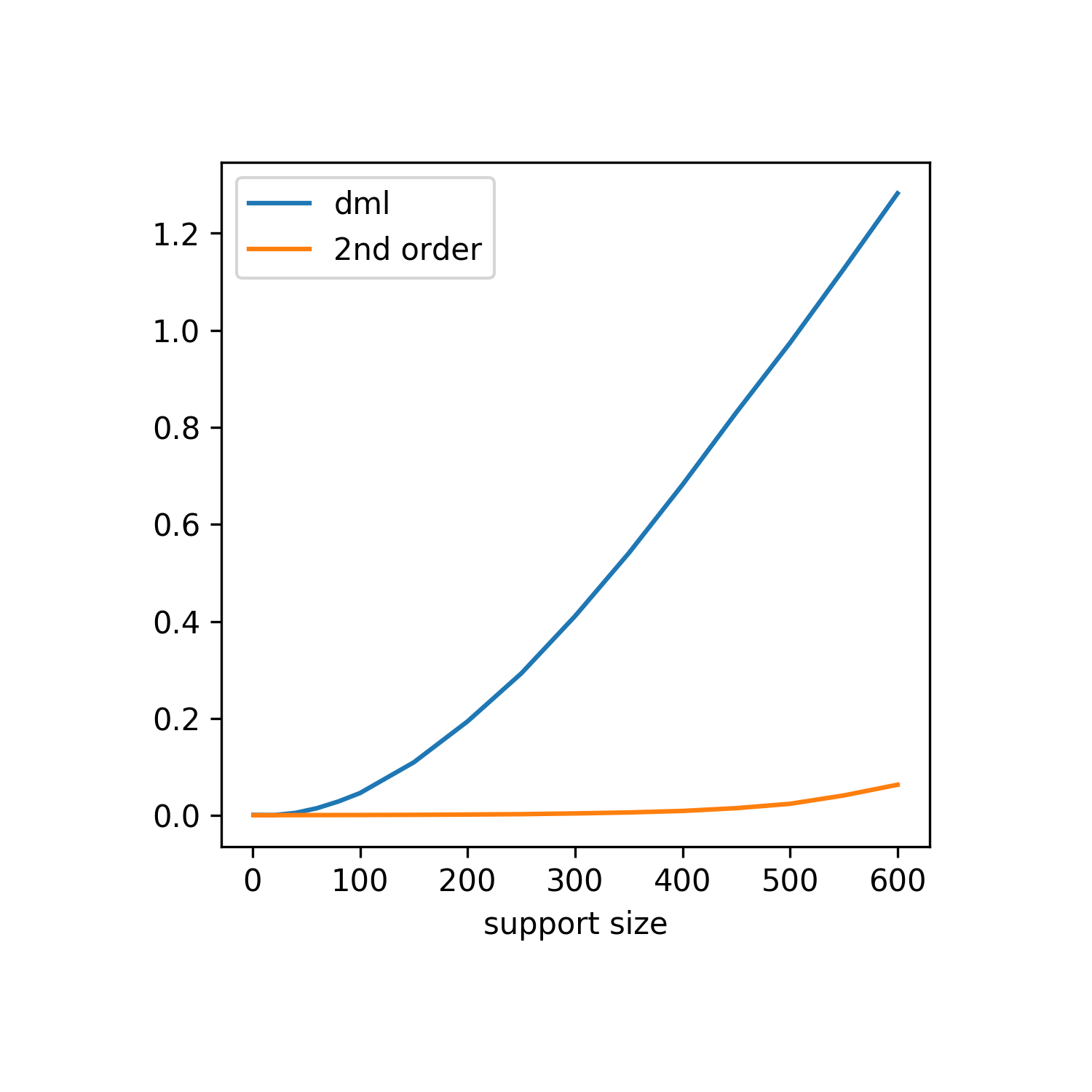}
\caption{\tiny{$n=5000, p=1000$}}
\end{subfigure}
~
\begin{subfigure}[t]{0.15\textwidth}
\centering
\includegraphics[scale=.23]{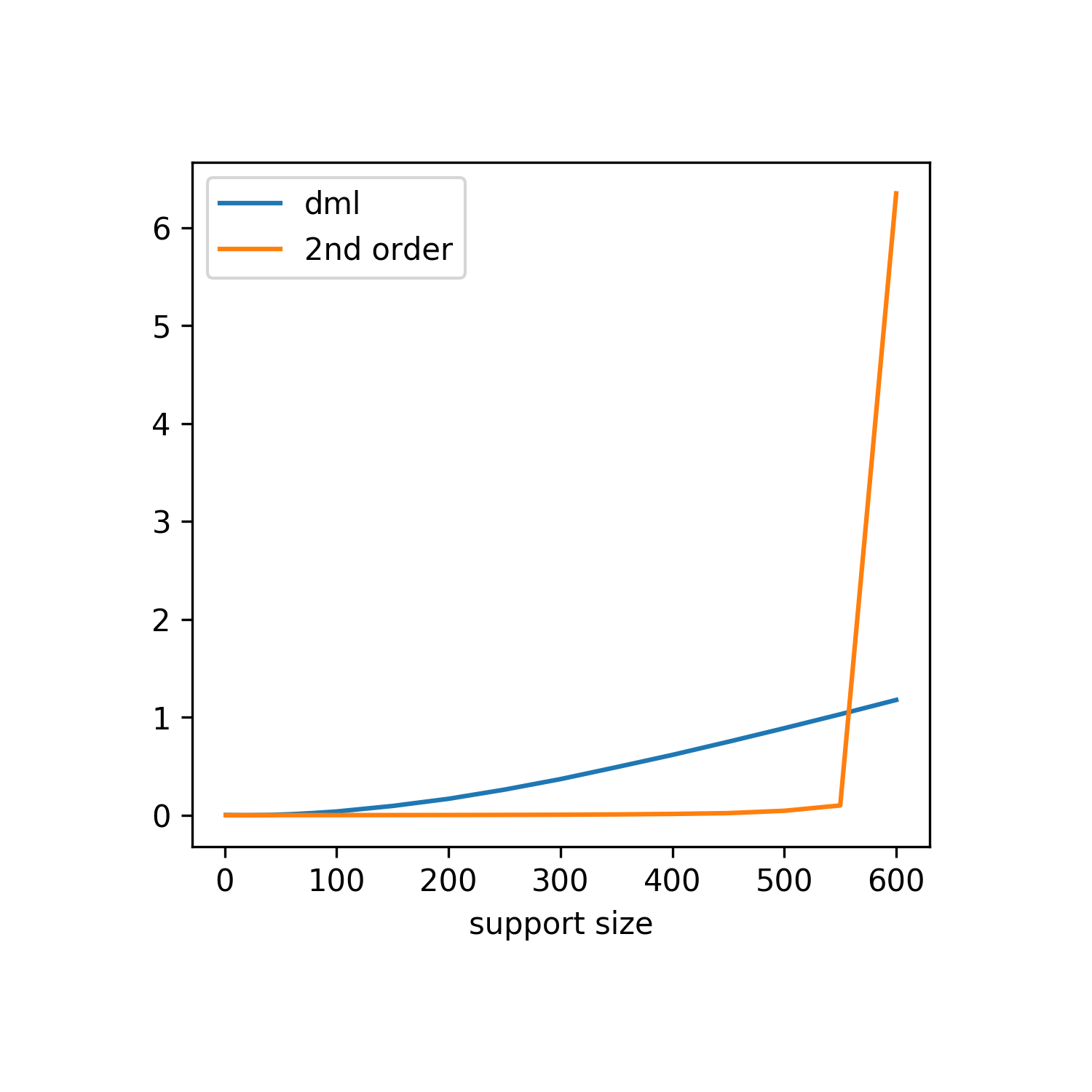}
\caption{\tiny{$n=5000, p=2000$}}
\end{subfigure}
\begin{subfigure}[t]{0.16\textwidth}
\centering
\includegraphics[scale=.23]{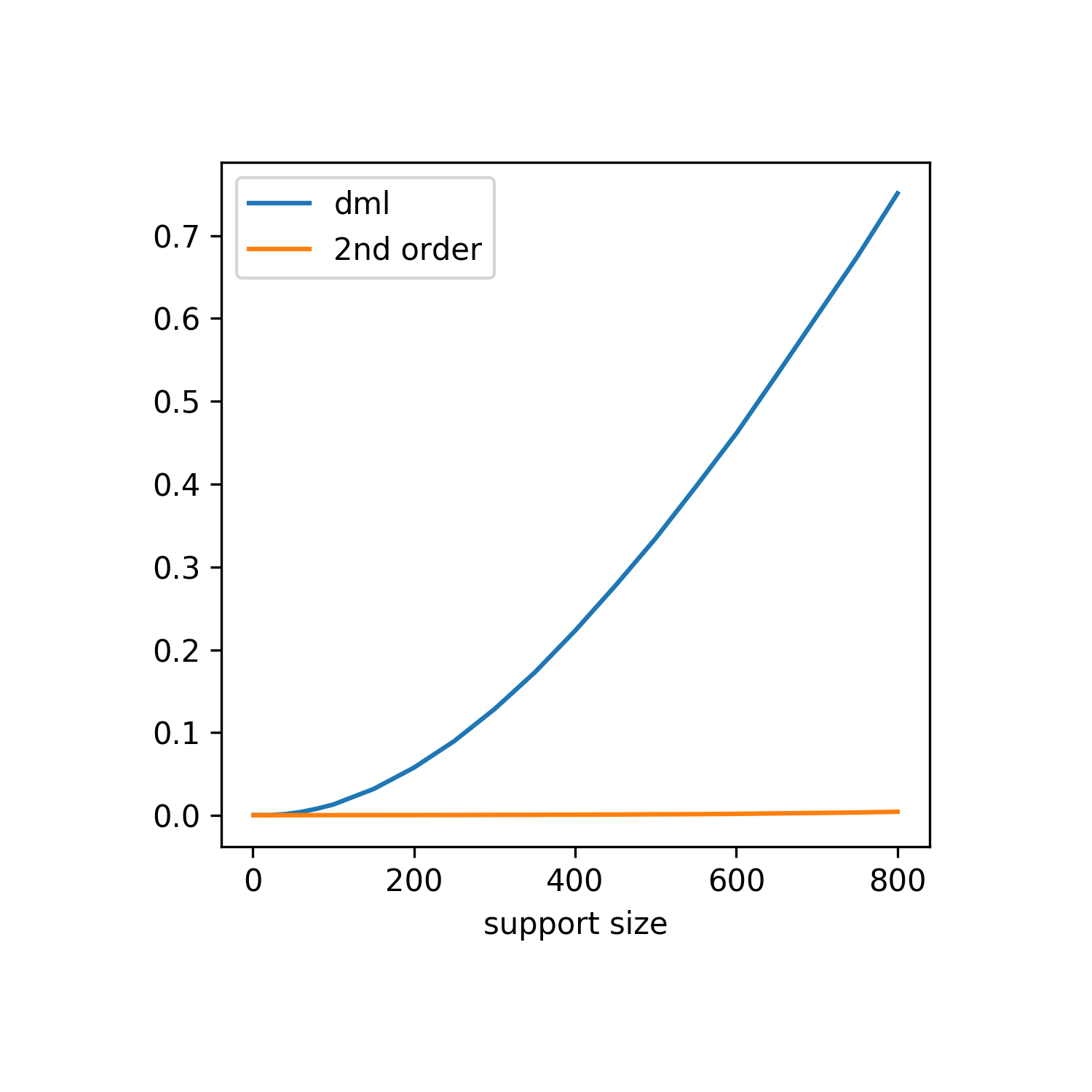}
\caption{\tiny{$n=10000, p=1000$}}
\end{subfigure}
\caption{\small{MSE of both estimators as the sparsity varies for different sample size and dimension pairs $(n,p)$. Note that the range of the support sizes is larger for larger $n$. $\sigma_\epsilon=1$.}}
\label{fig:comparison-mse}
\end{figure}

\paragraph{Varying $\sigma_\epsilon$} Finally \figref{comparison-mse-var} displays performance as the variance $\sigma_\epsilon$ of the noise $\epsilon$ grows. 

\begin{figure}[htpb]
\begin{subfigure}[t]{0.15\textwidth}
\centering
\includegraphics[scale=.23]{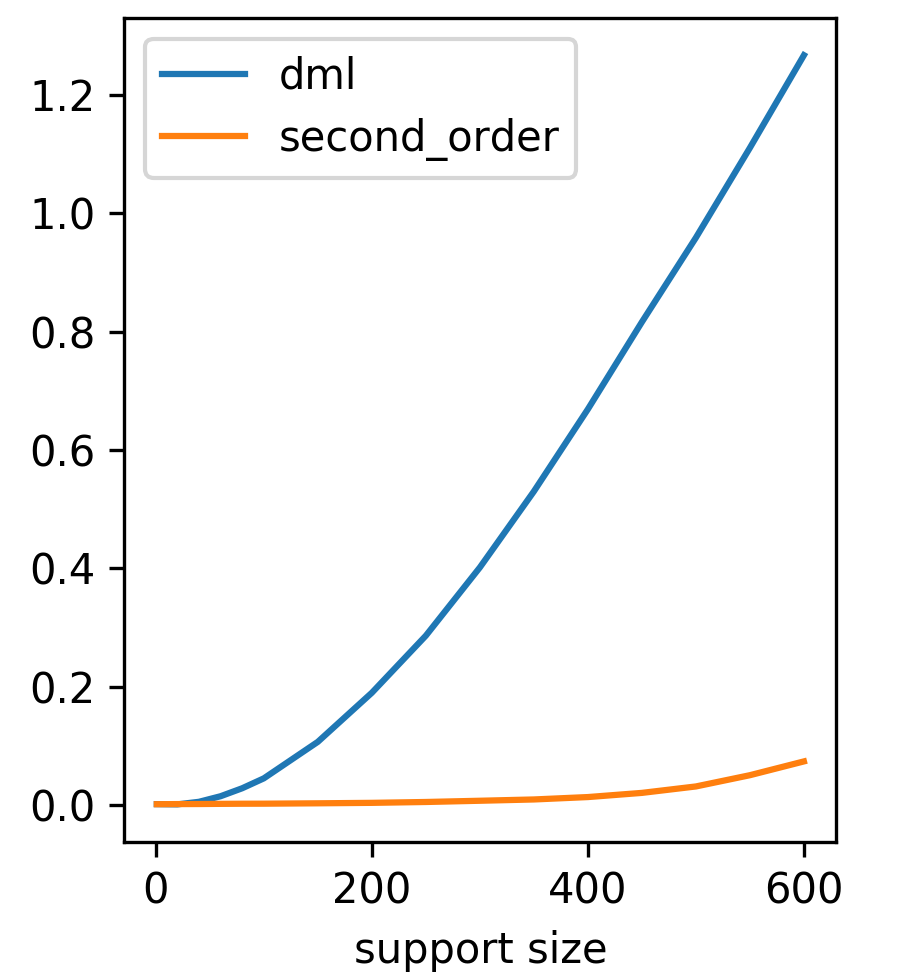}
\caption{\tiny{$\sigma_\epsilon=3$}}
\end{subfigure}
~
\begin{subfigure}[t]{0.15\textwidth}
\centering
\includegraphics[scale=.23]{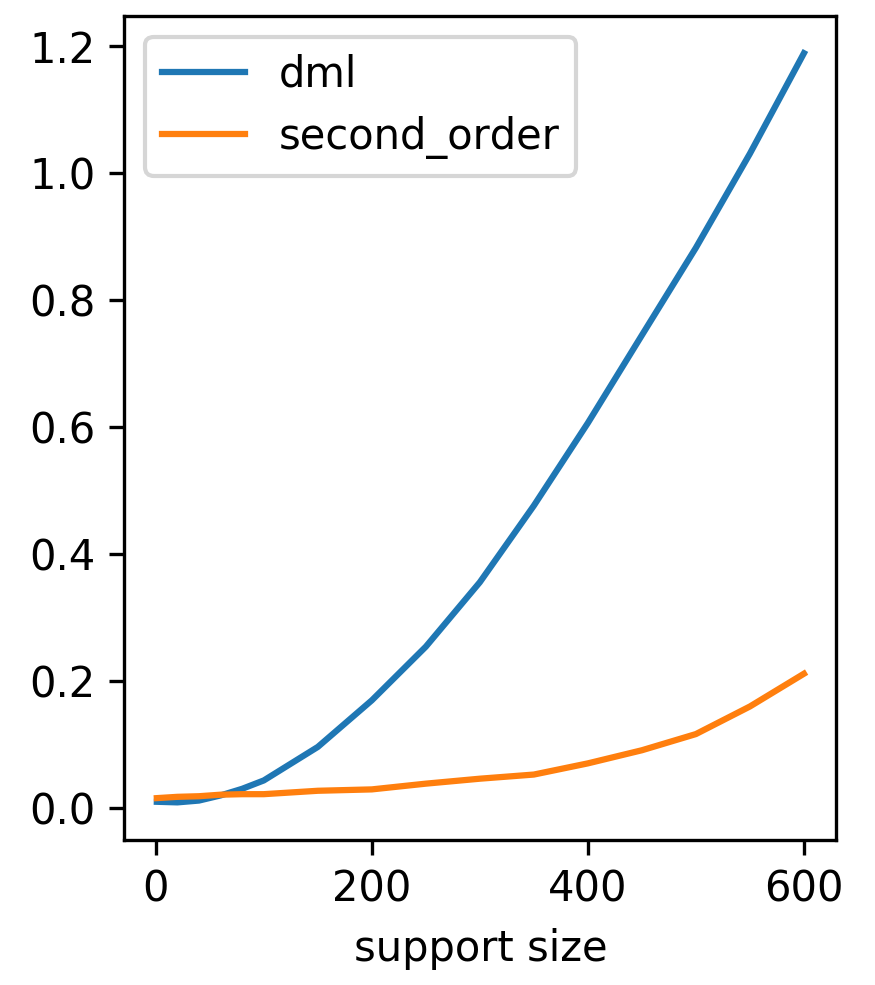}
\caption{\tiny{$\sigma_\epsilon=10$}}
\end{subfigure}
~
\begin{subfigure}[t]{0.15\textwidth}
\centering
\includegraphics[scale=.23]{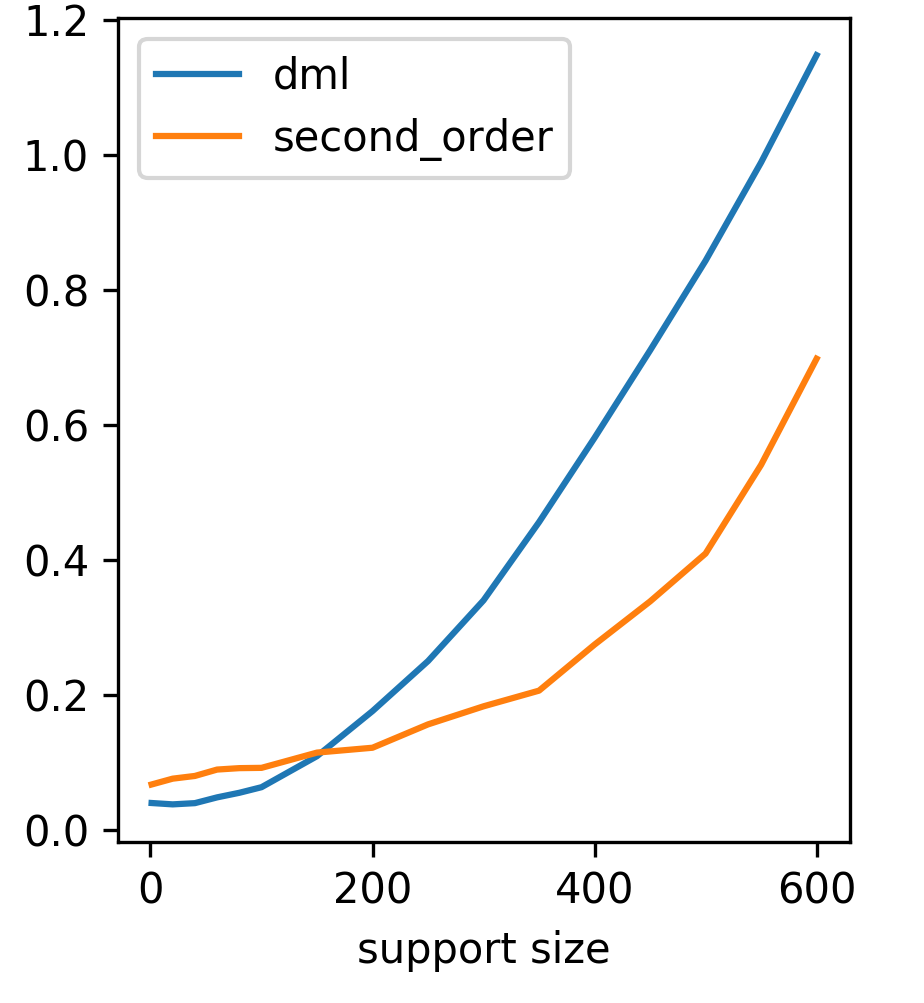}
\caption{\tiny{$\sigma_\epsilon=20$}}
\end{subfigure}
\caption{\small{MSE of both estimators as the sparsity varies for different variance parameters $\sigma_\epsilon$. $n=5000, p=1000$.}}
\label{fig:comparison-mse-var}
\end{figure}

\section{Conclusion}
Our aim in this work was to conduct accurate inference for fixed-dimensional target parameters in the presence of high-dimensional or nonparametric nuisance. 
To achieve this, we introduced a notion of $k$-th order orthogonal moments for two-stage $Z$-estimation, generalizing the first-order Neyman orthogonality studied in \citep{Chernozhukov2016}. Given $k$-th order orthogonal moments, we established that estimating nuisance at an $o(n^{-1/(2k+2)})$ rate suffices for $\sqrt{n}$-consistent and asymptotically normal estimates of target parameters. We then studied the PLR model popular in causal inference and showed that a valid second-order orthogonal moment exists if and only if the treatment residual is not normally distributed. In the high-dimensional linear nuisance setting, these explicit second-order orthogonal moments tolerate significantly denser nuisance vectors than those accommodated by \citep{Chernozhukov2016}. We complemented our results with synthetic demand estimation experiments showing the benefits of second-order orthogonal moments over standard Neyman-orthogonal moments.

\bibliography{ortho}
\bibliographystyle{icml2018}
\appendix
\onecolumn{

\section{Proof of Theorem \ref{thm:orth}} \label{sec:orth-proof}
We first prove the result for the sample-splitting estimator $\ssest$ in \eqnref{second-stage} and then discuss how to generalize for the $K$-fold cross fitting estimator $\cfest$ in \eqnref{cross} with $\sqrt{2n}$ scaling.

For each coordinate moment function $m_i$, the mean value theorem and the definition of $\ssest$ imply that
\begin{equation}\label{eqn:Taylor1}
 \frac{1}{n} \sum_{t=1}^n \inner{\nabla_\theta m_i(Z_t, \tilde{\theta}^{(i)}, \hat{h}(X_t))}{\theta_0-\ssest} 
 	= \frac{1}{n} \sum_{t=1}^n ( m_i(Z_t,\theta_0, \hat{h}(X_t)) - m_i(Z_t,\ssest, \hat{h}(X_t)))
	=\frac{1}{n} \sum_{t=1}^n m_i(Z_t,\theta_0, \hat{h}(X_t)) 
\end{equation}
for some convex combination, $\tilde{\theta}^{(i)}$, of $\ssest$ and $\theta_0$.
Hence,
\begin{align*} 
&\sqrt{n} (\theta_0-\ssest)
	\staticindic{\det \hat{J}(\hat{h})\not =0}
=\hat{J}(\hat{h})^{-1}\staticindic{\det \hat{J}(\hat{h})\not =0}\underbrace{\frac{1}{\sqrt{n}} \sum_{t=1}^n m(Z_t,\theta_0, \hat{h}(X_t))}_B
\\
&\text{for}
\quad\hat{J}(h) \defeq \frac{1}{n} \sum_{t=1}^n 
\begin{bmatrix}
\nabla_\theta m_1(Z_t, \tilde{\theta}^{(1)}, h(X_t)) \\
\cdots\\
\nabla_\theta m_d(Z_t, \tilde{\theta}^{(d)}, h(X_t)) \\
\end{bmatrix}
\in\R^{d\times d}.
\end{align*}

We will first show in \secref{J-convergence} that 
$\hat{J}(\hat{h})$ converges in probability to the invertible matrix $J = \E\left[\nabla_\theta m(Z,\theta_0, h_0(X))\right]$.
Hence, we will have $\staticindic{\det \hat{J}(\hat{h})\not =0} \toprob \staticindic{\det J\not =0} =1$ and $\hat{J}(\hat{h})^{-1}\staticindic{\det \hat{J}(\hat{h})\not =0} \toprob J^{-1}$ by the continuous mapping theorem \citep[Thm.\ 2.3]{vander}.
We will next show in \secref{B-convergence} that $B$ converges in distribution to a mean-zero multivariate Gaussian distribution with constant covariance matrix $V = \Cov(m(Z,\theta_0, h_0(X)))$. 
Slutsky's theorem \citep[Thm.\ 2.8]{vander} will therefore imply that $ \sqrt{n} (\theta_0-\ssest) \staticindic{\det \hat{J}(\hat{h})\not =0}$ converges in distribution to $N(0,J^{-1} VJ^{-1})$. Finally, the following lemma, proved in \secref{proof_conditionhighprob}, will imply that  $\sqrt{n} (\theta_0-\ssest)$ also converges in distribution to $N(0,J^{-1} VJ^{-1})$, as desired.

\begin{lemma}\label{lem:conditionhighprob}
Consider a sequence of binary random variables $Y_n \in \{0,1\}$ satisfying $Y_n \toprob  1$.
If $X_n Y_n \toprob  X$, then $X_n \toprob X$. 
Similarly, if $X_nY_n \todist X$, then $X_n \todist X$.
\end{lemma}

\subsection{Convergence of $\hat{J}(\hat{h}) - J$.} \label{sec:J-convergence}
For each coordinate $j$ and moment $m_i$ and $r>0$ defined in \assumpref{reg_moments}, the mean value theorem and Cauchy-Schwarz imply that
\begin{align*}
&\E\left[\left| \hat{J}_{ij}(\hat{h}) - \hat{J}_{ij}(h_0)\right|\staticindic{\tilde{\theta}^{(i)} \in \mc{B}_{\theta_0,r}} \mid \hat{h} \right] \staticindic{\hat{h} \in \mc{B}_{h_0,r}}\\
&~\leq \E\left[ 
	\left|\nabla_{\theta_j} m_i(Z_t, \tilde{\theta}^{(i)}, \hat{h}(X_t)) - 
	\nabla_{\theta_j} m_i(Z_t, \tilde{\theta}^{(i)}, h_0(X_t)) \right| \staticindic{\tilde{\theta}^{(i)} \in \mc{B}_{\theta_0,r}} \mid \hat{h} \right]\staticindic{\hat{h} \in \mc{B}_{h_0,r}} \\
	&~= \E\left[ \left| \inner{\hat{h}(X_t)-h_0(X_t)}{\grad_\gamma \nabla_{\theta_j} m_i(Z_t, \tilde{\theta}^{(i)}, \tilde{h}^{(j)}(X_t))} \right| \staticindic{\tilde{\theta}^{(i)} \in \mc{B}_{\theta_0,r}} \mid \hat{h} \right]\staticindic{\hat{h} \in \mc{B}_{h_0,r}} \\
	&~\leq \sqrt{\E\left[ \norm{\hat{h}(X_t)-h_0(X_t)}_2^2  \mid \hat{h} \right]\, 
		\  \sup_{h\in\mc{B}_{h_0,r}}\E\left[ \sup_{\theta\in\mc{B}_{\theta_0,r}}\norm{\grad_\gamma \nabla_{\theta_j} m_i(Z_t, \theta, h(X_t))}_2^2\right]}
\end{align*}
for $\tilde{h}^{(j)}(X_t)$ a convex combination of $h_0(X_t)$ and $\hat{h}(X_t)$.
The consistency of $\hat{h}$ (\assumpref{first-stage}) and the regularity condition \assumpref{gamma_theta_derivs}
therefore imply that $\E[| \hat{J}_{ij}(\hat{h}) - \hat{J}_{ij}(h_0)| \staticindic{\tilde{\theta}^{(i)} \in \mc{B}_{\theta_0,r}} \mid \hat{h} ]  \staticindic{\hat{h} \in \mc{B}_{h_0,r}} \toprob 0$
and hence that $| \hat{J}_{ij}(\hat{h}) - \hat{J}_{ij}(h_0)| \staticindic{\hat{h} \in \mc{B}_{h_0,r}, \tilde{\theta}^{(i)} \in \mc{B}_{\theta_0,r}}\toprob 0$
by the following lemma, proved in \secref{proof_conditional-moment-convergence}.

\begin{lemma}\label{lem:conditional-moment-convergence}
Consider a sequence of two random variables $X_n, Z_n$, where $X_n$ is a finite $d$-dimensional random vector. Suppose that
$\E\left[\|X_n\|_p^p | Z_n\right] \toprob 0$ for some $p\geq 1$. Then $X_n \toprob 0$.
\end{lemma}
Now \assumpsref{first-stage} and \assumpssref{first-stage-consistency} and the continuous mapping theorem imply that $\staticindic{\hat{h} \in \mc{B}_{h_0,r}} \toprob 1$.  Therefore, by Lemma \ref{lem:conditionhighprob}, we further have $| \hat{J}_{ij}(\hat{h}) - \hat{J}_{ij}(h_0)| \staticindic{\tilde{\theta}^{(i)} \in \mc{B}_{\theta_0,r}}  \toprob 0$.

The regularity \assumpsref{continuity} and \assumpssref{theta_dominated} additionally imply the uniform law of large numbers,
\begin{equation*}
\sup_{\theta\in \mc{B}_{\theta_0,r}} \|\textstyle\frac{1}{n} \sum_{t=1}^n
\nabla_\theta m_i(Z_t, \theta, h_0(X_t))- \E_{Z}[\nabla_\theta m_i(Z, \theta, h_0(X))]\|_2 \toprob 0
\end{equation*}
for each moment $m_i$ \citep[see, e.g.,][Lem. 2.4]{Newey1994}.
Taken together, these conclusions yield
$$\left[\hat{J}_i(\hat{h}) - \E_{Z}[\nabla_\theta m_i(Z, \tilde{\theta}^{(i)}, h_0(X))]\right]\staticindic{\tilde{\theta}^{(i)} \in \mc{B}_{\theta_0,r}}  \toprob 0,$$
for each $m_i$, where $\hat{J}_i(\hat{h})$ denotes the $i$-th row of $\hat{J}(\hat{h})$.

Since $\tilde{\theta}^{(i)}$ is a convex combination of $\ssest$ and $\theta_0$, the consistency of $\ssest$ implies that $\tilde{\theta}^{(i)} \toprob \theta_0$ and therefore that $\staticindic{\tilde{\theta}^{(i)} \in \mc{B}_{\theta_0,r}} \toprob 1$ and $\E_{Z}[\nabla_\theta m_i(Z, \tilde{\theta}^{(i)}, h_0(X))] \toprob \E_{Z}[\nabla_\theta m_i(Z, \theta_0, h_0(X))]$ by the continuous mapping theorem. 
Lemma \ref{lem:conditionhighprob} therefore implies that $\hat{J}_i(\hat{h}) - \E_{Z}[\nabla_\theta m_i(Z, \tilde{\theta}^{(i)}, h_0(X))]  \toprob 0$
and hence that $\hat{J}_i(\hat{h}) \toprob \E_{Z}[\nabla_\theta m_i(Z, \theta_0, h_0(X))]$, as desired.

\subsection{Asymptotic Normality of $B$.} \label{sec:B-convergence}
For a vector $\gamma\in \mathbb{R}^{\ell}$ and a vector $\alpha\in \mathbb{N}^{\ell}$, we define the shorthand ${\gamma}^{\alpha}\defeq \prod_{i=1}^\ell\gamma_\ell^{\alpha_\ell}$.

To establish the asymptotic normality of $B$, we let $k=\max_{\alpha \in S} \|\alpha\|_1$ and apply Taylor's theorem with $k+1$-order remainder around $h_0(X_t)$ for each $X_t$: 
\begin{equation}
\begin{aligned}
B = \underbrace{\frac{1}{\sqrt{n}} \sum_{t=1}^n m(Z_t, \theta_0, h_0(X_t))}_{C}  &+
\underbrace{ \frac{1}{\sqrt{n}} \sum_{t=1}^n \sum_{ \alpha: \alpha \in S } \frac{1}{\norm{\alpha}_1!} D^{\alpha} m(Z_t, \theta_0, h_0(X_t)) \left(\hat{h}(X_t) - h_0(X_t)\right)^{\alpha}}_{G}\\
&+ \underbrace{ \frac{1}{\sqrt{n}} \sum_{t=1}^n \sum_{ \alpha: \|\alpha\|_1 \leq k,\alpha \not  \in S } \frac{1}{\norm{\alpha}_1!} D^{\alpha} m(Z_t, \theta_0, h_0(X_t)) \left(\hat{h}(X_t) - h_0(X_t)\right)^{\alpha}}_{E}\\ 
&+ \underbrace{\frac{1}{\sqrt{n}} \sum_{t=1}^n \sum_{ \alpha: \|\alpha\|_1 = k+1 } \frac{1}{(k+1)!} \begin{bmatrix}
D^{\alpha} m_1(Z_t,\theta_0, \tilde{h}^{(1)}(X_t)) \\
\cdots\\
D^{\alpha} m_d(Z_t,\theta_0, \tilde{h}^{(d)}(X_t)) \\
\end{bmatrix}  \left(\hat{h}(X_t)-h_0(X_t)\right)^{\alpha}}_{F},
\end{aligned}
\end{equation}
where $\tilde{h}^{(i)}(X_t),i=1,2,\ldots,d$ are vectors which are (potentially distinct) convex combinations of $\hat{h}(X_t)$ and $h_0(X_t)$.
Note that $C$ is the sum of $n$ i.i.d.\ mean-zero random vectors divided by $\sqrt{n}$ and that the covariance $V=\Cov(m(Z,\theta_0, h_0(X)))$ of each vector is finite by \assumpref{reg_m_dominated}. Hence, the central limit theorem implies that $C\rightarrow_d N(0,V)$. It remains to show that $G, E, F\toprob 0$. 

First we argue that the rates of first stage consistency (\assumpref{first-stage})  imply that $E,F\toprob 0$. To achieve this we will show that $\E[|E_i|\mid \hat{h}], \E[|F_i|\mid \hat{h}] \toprob 0$, where $E_i$ and $F_i$ represent the $i$-th entries of $E$ and $F$ respectively. Since the number of entries $d$ is a constant, Lemma \ref{lem:conditional-moment-convergence} will then imply that $E, F\toprob 0$. First we have
\begin{align*}
\E [|E_i| \mid \hat{h}] \leq~&\sum_{ \alpha: \|\alpha\|_1 \leq k,\alpha \not  \in S } \frac{ \sqrt{n} }{\norm{\alpha}_1!} \E_{Z_t}[| D^{\alpha} m_i(Z_t, \theta_0, h_0(X_t)) (\hat{h}(X_t) - h_0(X_t))^{\alpha} | ] \tag{triangle inequality}\\
	\leq~&  \sum_{ \alpha: \|\alpha\|_1 \leq k,\alpha \not  \in S } \frac{ \sqrt{n} }{\norm{\alpha}_1!} \sqrt{\E[| D^{\alpha} m_i(Z_t, \theta_0, h_0(X_t))|^2]} \sqrt{\E_{X_t}[|\hat{h}(X_t) - h_0(X_t)|^{2\alpha}]} \tag{Cauchy-Schwarz}\\
	\leq~& \sum_{ \alpha: \|\alpha\|_1 \leq k,\alpha \not  \in S } \frac{ \sqrt{n} }{\norm{\alpha}_1!} \lambda_*(\theta_0, h_0)^{1/4} \sqrt{\E_{X_t}[|\hat{h}(X_t) - h_0(X_t)|^{2\alpha} ]} \tag{\assumpref{reg_bound_diff}}\\
	\leq~&  \max_{ \alpha: \|\alpha\|_1 \le q k,\alpha \not  \in S }  \lambda_*(\theta_0, h_0)^{1/4} \sqrt{n}\sqrt{\E_{X_t}[|\hat{h}(X_t) - h_0(X_t)|^{2\alpha} ]} \toprob 0.
\tag{\assumpref{first-stage}}
\end{align*}
Since $\tilde{h}^{(i)}$ is a convex combination of $\hat{h}$ and $h_0$, parallel reasoning yields
\begin{align*}
\E  [|F_i| \mid \hat{h} ] &\staticindic{\hat{h} \in \mc{B}_{h_0,r}} 
	\leq  \max_{ \alpha: \|\alpha\|_1 = k+1}  \staticindic{\hat{h} \in \mc{B}_{h_0,r}}\sqrt{\E_{Z_t}[| D^{\alpha} m_i(Z_t, \theta_0, \tilde{h}^{(i)}(X_t))|^2]} \sqrt{n} \sqrt{\E_{X_t}[|\hat{h}(X_t) - h_0(X_t)|^{2\alpha}]} \\
	\leq~&  \max_{ \alpha: \|\alpha\|_1 = k+1}  \lambda_*(\theta_0, h_0)^{1/4}  \sqrt{n} \sqrt{\E_{X_t}[|\hat{h}(X_t) - h_0(X_t)|^{2\alpha}]} \toprob 0
\tag{\assumpsref{reg_bound_diff} and \assumpssref{first-stage}}.
\end{align*}
As in \secref{J-convergence}, the consistency of $\hat{h}$ (\assumpref{first-stage}) further implies that $\E  [|F_i| \mid \hat{h} ]\toprob 0$.

Finally, we argue that orthogonality and the rates of the first stage imply that $G\toprob 0$. By $S$-orthogonality of the moments, for $\alpha \in S$, $\E \left[D^{\alpha}m(Z_t, \theta_0, h_0(X_t)) |X_t \right]=0$ and in particular
 \begin{equation}\label{eqn:firstD} 
 \E \left[  D^{\alpha} m(Z_t, \theta_0, h_0(X_t)) \left(\hat{h}(X_t) - h_0(X_t)\right)^{\alpha} | \hat{h}\right]=\E \left[ \E \left[D^{\alpha}m(Z_t, \theta_0, h_0(X_t)) |X_t \right] \left(\hat{h}(X_t) - h_0(X_t)\right)^{\alpha} | \hat{h}\right]=0.
\end{equation}
We now show that $\E \left[ G_i^2 | \hat{h} \right] \toprob 0$. We have
\begin{align*}
\E \left[G_i^2 | \hat{h} \right] =~& \frac{1}{n} \sum_{t,t'=1,2,\ldots,n, t\neq t'} \E\left[\sum_{ \alpha: \|\alpha\|_1 \leq k,\alpha \in S } \frac{1}{\norm{\alpha}_1!} D^{\alpha} {m_i}(Z_t,\theta_0, h_0(X_t)) \left(\hat{h}(X_t) - h_0(X_t)\right)^{\alpha} | \hat{h} \right]^2\\ 
&+ \frac{1}{n}\sum_{t=t'=1}^{n} \E\left[ \left( \sum_{ \alpha: \|\alpha\|_1 \leq k,\alpha \in S } \frac{1}{\norm{\alpha}_1!} D^{\alpha} {m_i}(Z_t, \theta_0, h_0(X_t)) \left(\hat{h}(X_t) - h_0(X_t)\right)^{\alpha}\right)^2 | \hat{h} \right]
\end{align*} All the cross terms are zero because of \eqnref{firstD}. Therefore:
\begin{align*}
{\E \left[ G_i^2 | \hat{h} \right]} =~& \E\left[ \left( \sum_{ \alpha: \alpha \in S } \frac{1}{\norm{\alpha}_1!} D^{\alpha} {m_i}(Z_t, \theta_0, h_0(X_t)) \left(\hat{h}(X_t) - h_0(X_t)\right)^{\alpha}\right)^2| \hat{h}  \right] \\
\leq~&  \E\left[ \sum_{ \alpha: \alpha \in S } \frac{1}{\norm{\alpha}_1!} \left(D^{\alpha} {m_i}(Z_t, \theta_0, h_0(X_t)) \left(\hat{h}(X_t) - h_0(X_t)\right)^{\alpha}\right)^2  | \hat{h} \right] \tag{Jensen's inequality}\\
\leq~&  \max_{ \alpha: \alpha \in S } \E\left[ \left(D^{\alpha} {m_i}(Z_t, \theta_0, h_0(X_t)) \left(\hat{h}(X_t) - h_0(X_t)\right)^{\alpha}\right)^2 | \hat{h} \right] \\
\leq~&  \max_{ \alpha: \alpha \in S } \sqrt{\E\left[ \left(D^{\alpha} {m_i}(Z_t, \theta_0, h_0(X_t))\right)^4\right]} \sqrt{\E\left[\left(\hat{h}(X_t) - h_0(X_t)\right)^{4\alpha} | \hat{h} \right]} \tag{Cauchy-Schwarz}\\
=~&  \max_{ \alpha:\alpha \in S }  \sqrt{\lambda_*(\theta_0,h_0)} \sqrt{\E\left[\left(\hat{h}(X_t) - h_0(X_t)\right)^{4\alpha} | \hat{h} \right]} \tag{\assumpref{reg_bound_diff}}
\end{align*}
{Given \assumpref{first-stage-consistency} we get that the latter converges to zero in probability. Given that the number of moments $d$ is also a constant, we have shown that $\E[\|G\|_2^2 | \hat{h}] \toprob 0$.  By Lemma \ref{lem:conditional-moment-convergence} the latter implies that $G\toprob 0$.}

The proof for the $K$-fold cross fitting estimator $\cfest$ follows precisely the same steps as the $\ssest$ proof (with $\sqrt{2n}$ scaling instead of $\sqrt{n}$ scaling) except for the final argument concerning $G \toprob 0$. In this case $G=\sum_{k=1}^KG_k$, where, for $k=1,\ldots,K$. 
$$G_k=\frac{1}{\sqrt{2n}}\sum_{t \in I_k} \sum_{ \alpha: \alpha \in S } \frac{1}{\norm{\alpha}_1!} D^{\alpha} m(Z_t, \theta_0, h_k(X_t)) \left(\hat{h}_k(X_t) - h_k(X_t)\right)^{\alpha}.$$ 
$K$ is treated as constant with respect to the other problem parameters, and therefore it suffices to show $G_k \toprob 0$, for all $k=1,2,\ldots,K$. Fix $k \in [K]$.  By Lemma \ref{lem:conditional-moment-convergence} it suffices to show $\mathbb{E}\left[G_k^2|\hat{h}_k\right] \toprob 0$. The proof of this follows exactly the same steps as proving $\mathbb{E}\left[G^2|\hat{h}\right] \toprob 0$ in the $\ssest$ case. The diagonal terms can be bounded in an identical way and the cross terms are zero again because $\hat{h}_k$ is trained in the first stage on data $(X_t)_{t \in I_k^c}$ and therefore the data $(X_t)_{t \in I_k}$ remain independent given $\hat{h}_k$.
Our proof is complete.

\section{Proof of \thmref{consistency}} \label{sec:proof-consistency}
We prove the result for the sample-splitting estimator $\ssest$ in \eqnref{second-stage}.
The proof for the $K$-fold cross fitting estimator $\cfest$ in \eqnref{cross} is analogous and follows as in \citep{Chernozhukov2016}.

Fix any compact $A \subseteq \Theta$.
Our initial goal is to establish the uniform convergence
\begin{align}\label{eqn:uniform_convergence}
\sup_{\theta\in A} |\textstyle\frac{1}{n} \sum_{t=1}^n m_i(Z_t,\theta, \hat{h}(X_t)) - \E[m_i(Z,\theta, h_0(X))] | \toprob 0
\end{align}
for each moment $m_i$.  To this end, we first note that the continuity (\assumpref{continuity}) and domination (\assumpref{reg_m_dominated}) of $m_i$ imply the uniform law of large numbers
\begin{equation*}
\sup_{\theta\in A, h \in\mc{B}_{h_0,r}}\textstyle |\textstyle\frac{1}{n} \sum_{t=1}^n
m_i(Z_t, \theta, h(X_t))- \E_{Z}[m_i(Z, \theta, h(X))]| \toprob 0
\end{equation*}
for each moment $m_i$ \citep[see, e.g.,][Lem. 2.4]{Newey1994}.
Moreover, the mean value theorem and two applications of Cauchy-Schwarz yield
\begin{align*}
|\E[m_i(Z,\theta, \hat{h}(X)) \mid \hat{h}] - \E[m_i(Z,\theta, h_0(X))] |
&\leq |\E[\inner{\grad_\gamma m_i(Z,\theta, \tilde{h}^{(i)}(X))}{\hat{h}(X) - h_0(X)} \mid \hat{h}] |\\
& \leq |\E[\|\grad_\gamma m_i(Z,\theta, \tilde{h}^{(i)}(X))\|_2 \|\hat{h}(X) - h_0(X)\|_2 \mid \hat{h}] \\
&\leq \sqrt{\E[\norm{\grad_\gamma m_i(Z,\theta, \tilde{h}^{(i)}(X))}_2^2 \mid \hat{h}] \E[\norm{\hat{h}(X) - h_0(X)}_2^2 \mid \hat{h}]}
\end{align*}
for $\tilde{h}^{(i)}$ a convex combination of $h_0$ and $\hat{h}$.
Hence, the uniform bound on the moments of $\grad_\gamma m_i$ (\assumpref{reg_grad_gamma}) and the consistency of $\hat{h}$ (\assumpref{first-stage-consistency}) imply $\sup_{\theta\in A} |\E[m_i(Z,\theta, \hat{h}(X)) \mid \hat{h}] - \E[m_i(Z,\theta, h_0(X))] | \toprob 0$, and therefore 
\begin{align*}
\staticindic{\hat{h} \in \mc{B}_{h_0,r}}
\sup_{\theta\in A} |\textstyle\frac{1}{n} \sum_{t=1}^n m_i(Z_t,\theta, \hat{h}(X_t)) - \E[m_i(Z,\theta, h_0(X))] | \toprob 0
\end{align*}
by the triangle inequality.
Since $\staticindic{\hat{h} \in \mc{B}_{h_0,r}}\toprob 1$ by the assumed consistency of $\hat{h}$, 
the uniform convergence \eqnref{uniform_convergence} follows from \lemref{conditionhighprob}.
Given the uniform convergence \eqnref{uniform_convergence}, standard arguments now imply consistency given identifiability (\assumpref{identifiable}) and either the compactness conditions of \assumpargref{consistency}{compactness} \citep[see, e.g.,][Thm. 2.6]{Newey1994} or the convexity conditions of \assumpargref{consistency}{convexity} \citep[see, e.g.,][Thm. 2.7]{Newey1994}. 

\section{Proof of Lemma \ref{thm:Sorth}}

We will use the inequality that for any vector of random variables $(W_1,\ldots, W_{K})$,
$$\textstyle\E \left[ \prod_{i=1}^{K} |W_i|\right] \leq \prod_{i=1}^{K} \E \left[ |W_i|^{K}\right]^{\frac{1}{K}},$$
which follows from repeated application of \Holder's inequality.  
In particular, we have 
\begin{align*}
\textstyle
\E_{X}\left[\prod_{i=1}^\ell \left|\hat{h}_i(X)-h_{0,i}(X)\right|^{2\alpha_i} \right] \leq \prod_{i=1}^\ell \E_{X}\left[\left|\hat{h}_i(X)-h_{0,i}(X)\right|^{2\|\alpha\|_1} \right]^{\nicefrac{\alpha_i}{\|\alpha\|_1}} = \prod_{i=1}^\ell \|\hat{h}_i - h_{0,i}\|_{2\|\alpha\|_1}^{2\alpha_i}
   \end{align*}
Thus the first part follows by taking the root of the latter inequality and multiplying by $\sqrt{n}$. For the second part of the lemma, observe that under the condition for each nuisance function we have:
\begin{align*}
\sqrt{n} \prod_{i=1}^\ell \|\hat{h}_i - h_{0,i}\|_{2\|\alpha\|_1}^{\alpha_i} =~& n^{\frac{1}{2}-\sum_{i=1}^\ell \frac{\alpha_i}{\kappa_i \|\alpha\|_1}} \prod_{i=1}^\ell \left( n^{\frac{1}{\kappa_i\|\alpha\|_1}}\|\hat{h}_i - h_{0,i}\|_{2\|\alpha\|_1}\right)^{\alpha_i}
\end{align*}
If $\frac{1}{2}-\sum_{i=1}^\ell \frac{\alpha_i}{\kappa_i \|\alpha\|_1}\leq 0$, then all parts in the above product converge to $0$ in probability.

For the second part for all $\alpha  \in S$ we similarly have \begin{align*}
\textstyle
\E_{X}\left[\prod_{i=1}^\ell \left|\hat{h}_i(X)-h_{0,i}(X)\right|^{4\alpha_i} \right] \leq \prod_{i=1}^\ell \E_{X}\left[\left|\hat{h}_i(X)-h_{0,i}(X)\right|^{4\|\alpha\|_1} \right]^{\nicefrac{\alpha_i}{4\|\alpha\|_1}} = \prod_{i=1}^\ell \|\hat{h}_i - h_{0,i}\|_{4\|\alpha\|_1}^{4\alpha_i}
\end{align*}Hence to satisfy \assumpref{first-stage-consistency} it suffices to satisfy $\forall \alpha \in S, \forall  i$, $\|\hat{h}_i - h_{0,i}\|_{4\|\alpha\|_1} \toprob 0$. But by Holder inequality and our hypothesis we have $$\|\hat{h}_i - h_{0,i}\|_{4\|\alpha\|_1} \leq \|\hat{h}_i - h_{0,i}\|_{4[\max_{\alpha \in S}\|\alpha\|_1]} \toprob 0,$$ as we wanted.

\section{Proof of \thmref{limitations}}
\label{sec:limitations-proof}
Suppose that the PLR model holds with the conditional distribution of $\eta$ given $X$ Gaussian.
Consider a generic moment $m(T,Y,\theta_0, f_0(X),g_0(X),h_0(X))$, where $h_0(X)$ represents any additional nuisance independent of $f_0(X), g_0(X)$.
We will prove the result by contradiction.
Assume that $m$ is $2$-orthogonal with respect to $(f_0(X), g_0(X))$ and satisfies Assumption~\ref{ass:main}.
By $0$-orthogonality, we have
\begin{align}\label{eqn:0-moment}
\E \left[ m(T,Y, \theta_0, f_0(X),g_0(X),h_0(X)) |X\right]=&0
\end{align} 
for any choice of true model parameters $(\theta_0, f_0, g_0, h_0)$, so 
\begin{equation*}
\nabla_{f_0(X)}\E \left[ m(T,Y, \theta_0, f_0(X),g_0(X),h_0(X)) |X \right]=\nabla_{g_0(X)}\E \left[ m(T,Y, \theta_0, f_0(X),g_0(X),h_0(X)) |X \right]=0.
\end{equation*} 
Since $m$ is continuously differentiable (\assumpref{continuity}), we may differentiate under the integral sign \citep{Flanders1973} to find that
\begin{align*}
0 &= \nabla_{f_0(X)}\E \left[ m(T,Y, \theta_0, f_0(X),g_0(X),h_0(X)) |X \right] \\
&= \nabla_{f_0(X)}\E \left[ m(T,\theta_0T + f_0(X) + \epsilon, \theta_0, f_0(X),g_0(X),h_0(X)) |X \right] \\
&= \E \left[ \nabla_2m(T,Y, \theta_0, f_0(X),g_0(X),h_0(X))+\nabla_4m(T,Y, \theta_0, f_0(X),g_0(X),h_0(X)) |X \right] \quad\text{and}\\
0 &= \nabla_{g_0(X)}\E \left[ m(T,Y, \theta_0, f_0(X),g_0(X),h_0(X)) |X \right] \\
&= \nabla_{g_0(X)}\E \left[ m(g_0(X)+\eta,\theta_0(g_0(X)+\eta) + f_0(X) + \eta, \theta_0, f_0(X),g_0(X),h_0(X)) |X \right] \\
& =\E \left[ \nabla_1m(T,Y, \theta_0, f_0(X),g_0(X),h_0(X))+\nabla_2m(T,Y, \theta_0, f_0(X),g_0(X),h_0(X))\theta_0 |X \right] \\
&+\E \left[ \nabla_5m(T,Y, \theta_0, f_0(X),g_0(X),h_0(X))  |X \right].
\end{align*}
Moreover, by $1$-orthogonality, we have $\E \left[ \nabla_{i} m(T,Y, \theta_0, f_0(X),g_0(X),h_0(X))  |X\right]=0$ for $i \in \{4,5\}$, so
\begin{equation}\label{eq:firstorder}
\E \left[ \nabla_{i} m(T,Y, \theta_0, f_0(X),g_0(X),h_0(X)) |X \right]=0,\quad \forall i \in \{1,2,4,5\} \quad\text{and}\quad \forall\,(\theta_0, f, g, h).
\end{equation}
Hence,
\begin{align*}
\nabla_{g_0(X)}\E \left[ \nabla_4 m(T,Y, \theta_0, f_0(X),g_0(X),h_0(X)) |X \right]=\nabla_{f_0(X)}\E \left[ \nabla_1 m(T,Y, \theta_0, f_0(X),g_0(X),h_0(X)) |X \right] = 0,
\end{align*}
and we again exchange derivative and integral using the continuity of $\grad^2 m$ (\assumpref{continuity}) \citep{Flanders1973} to find
\begin{align*} 
&\E \left[ \nabla_{1,4} m(T,Y, \theta_0, f_0(X),g_0(X),h_0(X))+\nabla_{5,4}m(T,Y, \theta_0, f_0(X),g_0(X),h_0(X))\right]\\
&+ \E\left[\theta_0\nabla_{2,4}m(T,Y, \theta_0, f_0(X),g_0(X),h_0(X)) |X  \right]\\
&=\E \left[ \nabla_{4,1} m(T,Y, \theta_0, f_0(X),g_0(X),h_0(X))+\nabla_{2,1} m(T,Y, \theta_0, f_0(X),g_0(X),h_0(X)) |X  \right].
\end{align*}
Since the partial derivatives of $m$ are differentiable by \assumpref{continuity}, we have $\nabla_{1,4}m=\nabla_{4,1}m$ and therefore 
\begin{align*}  
&\E \left[\nabla_{5,4}m(T,Y, \theta_0, f_0(X),g_0(X),h_0(X))  |X \right]
+ \theta_0\E\left[\nabla_{2,4}m(T,Y, \theta_0, f_0(X),g_0(X),h_0(X)) |X  \right] \\
&=\E \left[\nabla_{2,1} m(T,Y, \theta_0, f_0(X),g_0(X),h_0(X)) |X  \right]
\end{align*}
By $2$-orthogonality, $ \E \left[\nabla_{5,4}m(T,Y, \theta_0, f_0(X),g_0(X),h_0(X)) |X \right]=0$, and hence
\begin{align} \label{eq:12} 
\theta_0\E\left[\nabla_{2,4}m(T,Y, \theta_0, f_0(X),g_0(X),h_0(X)) |X  \right] 
=\E \left[\nabla_{2,1} m(T,Y, \theta_0, f_0(X),g_0(X),h_0(X)) |X  \right].
\end{align}

Note that equality \eqref{eq:firstorder} also implies
\begin{align*}
0=\nabla_{f_0(X)}\E \left[ \nabla_2 m(T,Y, \theta_0, f_0(X),g_0(X),h_0(X)) |X \right]=\nabla_{f_0(X)}\E \left[ \nabla_4 m(T,Y, \theta_0, f_0(X),g_0(X),h_0(X)) |X \right].
\end{align*}
We again exchange derivative and integral using the continuity of $\grad^2 m$ (\assumpref{continuity}) \citep{Flanders1973} to find
\begin{align}\label{eq:zero} 
0&=\E \left[ \nabla_{2,2} m(T,Y, \theta_0, f_0(X),g_0(X),h_0(X))+\nabla_{2,4}m(T,Y, \theta_0, f_0(X),g_0(X),h_0(X))| X\right]\\
&=\E \left[ \nabla_{4,2} m(T,Y, \theta_0, f_0(X),g_0(X),h_0(X))+\nabla_{4,4} m(T,Y, \theta_0, f_0(X),g_0(X),h_0(X)) |X  \right]. \notag 
\end{align}
Since the partial derivatives of $m$ are continuous by \assumpref{continuity}, we have $\nabla_{2,4}m=\nabla_{4,2}m$ and therefore 
\begin{align*}  
\E \left[\nabla_{2,2}m(T,Y, \theta_0, f_0(X),g_0(X),h_0(X))  |X \right]=\E \left[\nabla_{4,4} m(T,Y, \theta_0, f_0(X),g_0(X),h_0(X)) |X  \right]
\end{align*}
By $2$-orthogonality, $ \E \left[\nabla_{4,4}m(T,Y, \theta_0, f_0(X),g_0(X),h_0(X)) |X \right]=0$, and hence
\begin{align} \label{eqn:12} 
\E\left[\nabla_{2,2}m(T,Y, \theta_0, f_0(X),g_0(X),h_0(X)) |X  \right] 
=0
\end{align} Combining the equalities \eqref{eq:12}, \eqref{eq:zero}, and \eqref{eqn:12} we find that
\begin{align} \label{eq:1234} 
\E\left[\nabla_{2,1}m(T,Y, \theta_0, f_0(X),g_0(X),h_0(X)) |X  \right] 
=0.
\end{align}
Now, the $0$-orthogonality condition \eqnref{0-moment}, the continuity of $\grad m$ (\assumpref{continuity}), and differentiation under the integral sign \cite{Flanders1973} imply that
\begin{align*}
0 &= \nabla_{\theta_0}  \E \left[ m(T,Y, \theta_0, f_0(X),g_0(X),h_0(X)) |X\right] 
= \nabla_{\theta_0}  \E \left[ m(T,\theta_0 T + f_0(X) + \epsilon, \theta_0, f_0(X),g_0(X),h_0(X)) |X\right] \\
&= \E \left[  \nabla_2 m(T,Y, \theta_0, f_0(X),g_0(X),h_0(X)) \cdot T+\nabla_3 m(T,Y, \theta_0, f_0(X),g_0(X),h_0(X)) |X \right].
\end{align*}
Since $T=g_0(X)+\eta$ and $\E \left[  \nabla_2 m(T,Y, \theta_0, f_0(X),g_0(X),h_0(X))|X \right]=0$  by equality \ref{eqn:0-moment}, 
\begin{equation}\label{eq:12345}
\E \left[  \nabla_2 m(T,Y, \theta_0, f_0(X),g_0(X),h_0(X)) \cdot \eta+\nabla_3 m(T,Y, \theta_0, f_0(X),g_0(X),h_0(X)) |X \right]=0
\end{equation}
Since $\eta$ is conditionally Gaussian given $X$, Stein's lemma \citep{Stein}, the symmetry of the partial derivatives of $m$, and the equality \ref{eq:1234} imply that
\begin{align*}
&\E \left[  \nabla_2 m(T,Y, \theta_0, f_0(X),g_0(X),h_0(X)) \cdot \eta |X \right]=\E \left[  \nabla_2 m(g_0(X)+\eta,Y, \theta_0, f_0(X),g_0(X),h_0(X)) \cdot \eta |X \right]\\
&= \E \left[  \nabla_{\eta,2} m(g_0(X)+\eta,Y, \theta_0, f_0(X),g_0(X),h_0(X)) |X \right] \\
&=\E \left[  \nabla_{1,2} m(T,Y, \theta_0, f_0(X),g_0(X),h_0(X))  |X \right]
=\E \left[  \nabla_{2,1} m(T,Y, \theta_0, f_0(X),g_0(X),h_0(X))  |X \right]=0.
\end{align*}
Hence the equality \eqref{eq:12345} gives
$ \E \left[ \nabla_3 m(T,Y, \theta_0, f_0(X),g_0(X),h_0(X)) |X \right]=0$
which contradicts \assumpref{reg_full_rank}.

\section{Proof of \propref{deg}}
\label{sec:lemma_degeneracy}
Fix any moment of the form $m(T,Y,\theta,f(X),g(X),h(X))$, where $h$ represents any nuisance in addition to $(f,g)$. Let $F$ be the space of all valid nuisance functions $(f,g,h)$ and $F(X) = \{(f(X),g(X),h(X)) : (f,g,h) \in F\}$.

We prove the lemma by contradiction. Suppose $m$ satisfies the three hypothesis of our lemma. 
We start by establishing that $\Var\left(m(T,Y, \theta_0, f_0(X),g_0(X),h_0(X))\right)=0$ for all $(\theta_0,f_0,g_0,h_0)$. Fix any $(\theta_0,f_0,g_0,h_0)$, and suppose $\Var\left(m(T,Y, \theta_0, f_0(X),g_0(X),h_0(X))\right) > 0$. As in the beginning of the proof of Theorem~\ref{thm:orth} the mean value theorem implies
\begin{equation}\label{eqn:firsteq}
\hat{J}(\hat{f},\hat{g},\hat{h})\sqrt{n} (\theta_0-\ssest)
=\underbrace{\frac{1}{\sqrt{n}} \sum_{t=1}^n m\left(T_t,Y_t, \theta_0, \hat{f}(X_t),\hat{g}(X_t), \hat{h}(X_t)\right)}_B
\end{equation}where $\hat{J}(f,g,h) \defeq \frac{1}{n} \sum_{t=1}^n \nabla_\theta m(T_t,Y_t, \tilde{\theta}, f(X_t),g(X_t),h(X_t)),$  for some $\tilde{\theta}$ which is a convex combination of $\ssest,\theta_0$. In the proof of Theorem~\ref{thm:orth} we only use \assumpref{reg_full_rank} to invert $J = \E\left[\nabla_\theta m(T,Y, \theta_0, f_0(X),g_0(X),h_0(X)\right]$ which is the in-probability limit of $\hat{J}(\hat{f},\hat{g},\hat{h})$. In particular, both of the following results established in the proof of the Theorem~\ref{thm:orth} remain true in our setting:
\begin{itemize}
\item  $B$ tends to a normal distribution with mean zero and variance $\Var\left(m(T,Y, \theta_0, f_0(X),g_0(X),h_0(X))\right) > 0$.
\item  $\hat{J}(\hat{f},\hat{g},\hat{h})$ converges in probability to $J$.
\end{itemize} Since in this case $J=0$, as Assumption~\ref{ass:main}.3 is violated, and $\sqrt{n} (\theta_0-\ssest)$ is bounded in probability, we get that $\hat{J}(\hat{f},\hat{g},\hat{h})\sqrt{n} (\theta_0-\ssest)$ converges to zero in probability. By \eqnref{firsteq}, this contradicts the fact that $B$ converges to a distribution with non-zero variance. 
Hence, $\Var\left(m(T,Y, \theta_0, f_0(X),g_0(X),h_0(X))\right)=0$ as desired.

Now recalling that, for all $(\theta_0,f_0,g_0,h_0)$, $\mathbb{E}[m(T,Y, \theta_0, f_0(X),g_0(X),h_0(X))]=0$, we conclude that for all $(\theta_0,f_0,g_0,h_0)$, $m(T,Y, \theta_0, f_0(X),g_0(X),h_0(X))=0$, almost surely with respect to the random variables $X,\epsilon,\eta$. Now fix $(\theta_0,f_0,g_0,h_0)$. 
Now suppose that for some $(a,b)\in \R^2$, $m(a,b, \theta_0, f_0(X),g_0(X),h_0(X))\neq 0$.  
Then, since $m$ is continuous, there exists a neighborhood $\mc{N}$ such that $m(a',b',\theta_0, f_0(X),g_0(X),h_0(X))\neq 0$ for all $(a',b')\in\mc{N}$.
Since the conditional distribution of $\epsilon,\eta$ has full support (a.s.\ X) and, given $X$,  $(T,Y)$ is an invertible linear function of $(\epsilon, \eta)$, the conditional distribution of $(T,Y)$ given $X$ also has full support on $\R^2$ (a.s.\ X).
Hence, $\Pr(m(T,Y, \theta_0, f_0(X),g_0(X),h_0(X)) \neq 0) \geq \Pr((T,Y) \in \mc{N}) > 0$.
This is a contradiction as $m(T,Y, \theta_0, f_0(X),g_0(X),h_0(X))$ is a.s. zero.
Therefore, for almost every $X$ and all $a,b \in \mathbb{R}$ and $(\theta_0,f_0,g_0,h_0)$, $m(a,b, \theta_0, f_0(X),g_0(X),h_0(X))=0$. 
Since the distribution of $X$ is independent of $\theta_0$ and $|\Theta| \geq 2$, we therefore have 
\[
\E[m(Y,T, \theta, f_0(X),g_0(X),h_0(X))] = 0
\]
for some $\theta\neq\theta_0$, which contradicts identifiability.

\section{Proof of Lemma \ref{MGFgaussian}}\label{sec:proof-MGFgaussian}
Since the characteristic function of a Gaussian distribution is well-defined and finite on the whole real line, Levy's Inversion Formula implies that the Gaussian distribution is uniquely characterized by its moments \citep[Sec. 3.3.1]{Durrett}.

\section{Proof of Theorem \ref{thm:h1}}\label{sec:proof-h1}
Smoothness follows from the fact that $m$ is a polynomial in $(\theta, q(X), g(X), \mu_{r-1}(X))$.
Non-degeneracy follows from the PLR equations (\defref{plr}), the property $\E[\eta \mid X]=0$, and our choice of $r$ as
\begin{align*}
\E [\nabla_{\theta} m(Z, \theta_0, q_0(X),g_0(X),\E[\eta^{r-1}|X])]&
=-\E [ (T-g_0(X))(\eta^r-\E[\eta^r|X]- r\eta \E[\eta^{r-1}|X])]\\
&=-\E [ \eta(\eta^r-\E[\eta^r|X]- r\eta \E[\eta^{r-1}|X] ]\\
&=-\E[\E [ \eta^{r+1}|X]-r \E[\eta^2|X] \E [ \eta^{r-1}|X]] \not =0.
\end{align*}

We next establish $0$-orthogonality using the property $\E[\epsilon \mid X, T]=0$ of \defref{plr}:
\begin{equation*} 
\E\left[ m(Z, \theta_0, q_0(X),g_0(X),\E[\eta^{r-1}|X])\mid X \right]=\E[\epsilon\left(\eta^r-\E[\eta^r|X]- r\eta \E[\eta^{r-1}|X]\mid X\right)]=0.
\end{equation*}
Our choice of $r$ further implies identifiability as, for $\theta\neq\theta_0$,
\begin{align*} 
\E\left[ m(Z, \theta, q_0(X),g_0(X),\E[\eta^{r-1}|X])\right]
&=(\theta_0-\theta)\E[\E[\eta^{r+1}| X]-\E[\eta|X]\E[\eta^r|X]- rE[\eta^2|X] \E[\eta^{r-1}|X]] \\
&=(\theta_0-\theta)\E[\E[\eta^{r+1}| X] - rE[\eta^2|X] \E[\eta^{r-1}|X]]
\neq 0.
\end{align*}
We invoke the properties $\E[\eta \mid X]=0$ and $\E[\epsilon \mid X, T]=0$ of \defref{plr} to derive $1$-orthogonality via
\begin{align*} 
&\E \left[ \nabla_{q(X)} m(Z, \theta_0, q_0(X),g_0(X),\E[\eta^{r-1}|X]) |X\right]=-\E \left[ \eta^r-\E[\eta^r|X]- r\eta \E[\eta^{r-1}|X] \mid X\right]=0,\\
&\E \left[ \nabla_{g(X)} m(Z, \theta_0, q_0(X),g_0(X),\E[\eta^{r-1}|X]) |X\right] \\
&=\theta_0\E\left[\eta^r-\E[\eta^r|X]- r\eta \E[\eta^{r-1}|X]\mid X\right]-\E \left[ \epsilon(r\eta^{r-1}- r \E[\eta^{r-1}|X]) \mid X\right]=0, \qtext{and} \\
&\E\left[ \nabla_{\mu_{r-1}(X)} m(Z, \theta_0, q_0(X),g_0(X),\E[\eta^{r-1}|X]) \right]
=  -\E[\epsilon\, r\, \eta |X] = 0.
\end{align*}
The same properties also yield $2$-orthogonality for the second partial derivatives of $q(X)$ via
\begin{align*} 
&\E \left[ \nabla_{q(X),q(X)}^2 m(Z, \theta_0, q_0(X),g_0(X),\E[\eta^{r-1}|X]) |X\right]=0,\\
&\E \left[ \nabla_{q(X),g(X)}^2 m(Z, \theta_0, q_0(X),g_0(X),\E[\eta^{r-1}|X]) |X\right]=\E \left[ r\eta^{r-1}- r \E[\eta^{r-1}|X] |X\right]=0, \qtext{and}\\ 
&\E \left[ \nabla_{q(X),\mu_{r-1}(X)}^2 m(Z, \theta_0, q_0(X),g_0(X),\E[\eta^{r-1}|X]) |X\right]=\E \left[ r\,\eta |X\right]=0,
\end{align*}
for the second partial derivatives of $g(X)$ via
\begin{align*}
&\E \left[ \nabla_{g(X),g(X)}^2 m(Z, \theta_0, q_0(X),g_0(X),\E[\eta^{r-1}|X]) |X\right] 
=\E\left[ -(r\eta^{r-1}- r \E[\eta^{r-1}|X])+ \epsilon\,r(r-1)\eta^{r-2} |X\right]\\
&=r(r-2) \E \left[\E \left[ \epsilon | X,T \right] \eta^{r-2} |X\right]=0 \qtext{and} \\
&\E \left[ \nabla_{g(X),\mu_{r-1}(X)}^2 m(Z, \theta_0, q_0(X),g_0(X),\E[\eta^{r-1}|X]) |X\right]
= -\theta_0\E[ r \eta | X] + \E[\epsilon\, r| X] = 0,
\end{align*}
and for the second partial derivatives of $\mu_{r-1}(X)$ via
\begin{align*} 
&\E \left[ \nabla_{\mu_{r-1}(X), \mu_{r-1}(X)}^2 m(Z, \theta_0, q_0(X),g_0(X),\E[\eta^{r-1}|X]) \mid X\right]=0.
\end{align*}
This establishes $2$-orthogonality.

\section{Proof of Theorem \ref{thm:unknown2}} \label{sec:proof-unknown2}
The majority of the proof is identical to that of \thmref{h1}; it only remains to show that the advertised partial derivatives with respect to $\mu_{r}(X)$ are also mean zero given $X$.  These equalities follow from the property $\E[\eta \mid X]=0$ of \defref{plr}:
\begin{align*}
&\E \left[ \nabla_{\mu_{r}(X)} m(Z, \theta_0, q_0(X),g_0(X),\E[ \eta^{r-1} | X], \E[ \eta^{r} | X])) |X\right]=-\E \left[ \epsilon|X\right]= 0,\\
&\E \left[ \nabla_{\mu_{r}(X),\mu_{r}(X)}^2m(Z, \theta_0, q_0(X),g_0(X),\E[ \eta^{r-1} | X], \E[ \eta^{r} | X])) \mid X\right]= 0, \qtext{and}\\
&\E \left[ \nabla_{\mu_{r}(X),\mu_{r-1}(X)}^2m(Z, \theta_0, q_0(X),g_0(X),\E[ \eta^{r-1} | X], \E[ \eta^{r} | X])) \mid X\right]= 0.\\
\end{align*}

\section{Proof of Theorem \ref{thm:sparsethm}} \label{sec:proof-sparsethm}
We prove the result explicitly for the excess kurtosis setting with $\E[\eta^4]\not = 3\E[\eta^2]^2$.
A parallel argument yields the result for non-zero skewness ($\E[\eta^3]\not = 0$).

To establish $\sqrt{n}$-consistency and asymptotic normality, it suffices to check each of the preconditions of Theorems \ref{thm:orth} and \ref{thm:consistency}.
Since $\eta$ is independent of $X$ and $\E[\eta^4]\not = 3\E[\eta^2]^2$, the conditions of Theorem \ref{thm:unknown2} are satisfied with $r = 3$. 
Hence, the moments $m$ of Theorem \ref{thm:unknown2} satisfy $S$-orthogonality (\assumpref{orthogonality}) for $S=\{\alpha \in \mathbb{N}^4 : \|\alpha\|_1 \leq 2\}\setminus \{(1,0,0,1),(0,1,0,1)\} $ with respect to the nuisance $(\inner{q_0}{X},\inner{\gamma_0}{X},\E[\eta^2],\E[\eta^3])$, identifiability (\assumpref{identifiable}), non-degeneracy of $\E\left[\nabla_\theta m(Z,\theta_0,h_0(X))\right]$ (\assumpref{reg_full_rank}), and continuity of $\grad m^2$ (\assumpref{continuity}).
The form of $m$, the standard Gaussian i.i.d.\ components of $X$, and the almost sure boundedness of $\eta$ and $\epsilon$ further imply that the regularity conditions of \assumpref{reg_moments} are all satisfied for any choice of $r>0$. 
Hence, it only remains to establish the first stage consistency and rate assumptions (\assumpsref{first-stage-consistency} and \assumpssref{first-stage}) and the convexity conditions (\assumpargref{consistency}{convexity}).
  
\subsection{Checking Rate of First Stage (\assumpref{first-stage})}
We begin with \assumpref{first-stage}. Since $\{\alpha \in \mathbb{N}^4:\|\alpha\|_1 \leq 3\}\setminus S=\{\alpha \in \mathbb{N}^4:\|\alpha\|_1= 3\} \cup \{(1,0,0,1),(0,1,0,1)\}$ by Lemma \ref{thm:Sorth}, it suffices to establish the sufficient condition \eqnref{suff-cond-1} for $\alpha=(0,1,0,1)$ and $\alpha=(1,0,0,1)$ and the condition \eqnref{suff-cond-2} for the $\alpha$ with $\|\alpha\|_1=3$. 
Hence, it suffices to satisfy
\begin{itemize}
\item[(1)] $n^{\frac{1}{2}}\E_X[|\ldot{X}{\hat{q}-q_0}|^4]^{\frac{1}{4}} \cdot |\hat{\mu}_{3}-\E[\eta^3]| \toprob0$, which corresponds to $\alpha=(1,0,0,1)$ and condition \eqnref{suff-cond-1},
\item[(2)] $n^{\frac{1}{2}}\E_X[|\ldot{X}{\hat{\gamma}-\gamma_0}|^4]^{\frac{1}{4}} \cdot |\hat{\mu}_{3}-\E[\eta^3]|\toprob0$, which corresponds to $\alpha=(0,1,0,1)$ and condition \eqnref{suff-cond-1},
\item[(3)]  $ n^{\frac{1}{2}}\E_X[|\ldot{X}{\hat{q}-q_0}|^6]^{\frac{1}{2}}\toprob0$,

\item[(4)] $n^{\frac{1}{2}}\E_X[|\ldot{X}{\hat{\gamma}-\gamma_0}|^6]^{\frac{1}{2}} \toprob0$,

\item[(5)] $n^{\frac{1}{2}}|\hat{\mu}_{2}-\E[\eta^2]|^3 \toprob0$, and
\item[(6)] $n^{\frac{1}{2}}|\hat{\mu}_{3}-\E[\eta^3]|^3 \toprob0$,
\end{itemize} where $X$ a vector of i.i.d.\ mean-zero standard Gaussian entries, independent from the first stage, and the convergence to zero is considered in probability with respect to the first stage random variables.

We will estimate $q,\gamma_0$ using half of our first-stage sample and use our estimate $\hat{\gamma}$ to produce an estimate of the second and third moments of $\eta$ based on the other half of the sample and the following lemma.
\begin{lemma}\label{lem:proofstages}
Suppose that an estimator $\hat{\gamma} \in \mathbb{R}^p$ based on $n$ sample points satisfies $\E_X[|\ldot{X}{\hat{\gamma}-\gamma_0}|^6]^{\frac{1}{2}} =O_P\left( \frac{1}{\sqrt{n}}\right)$ for $X$ independent of $\hat{\gamma}$.  
If
\[\textstyle
\hat{\mu}_{2} := \frac{1}{n}\sum_{t=1}^n(T_t' - \inner{X_t'}{\hat{\gamma}})^2
\qtext{and}
\hat{\mu}_{3} := \frac{1}{n}\sum_{t=1}^n(T_t' - \inner{X_t'}{\hat{\gamma}})^3 - 3\frac{1}{n}\sum_{t=1}^n(T_t' - \inner{X_t'}{\hat{\gamma}})\hat{\mu}_{2}
\]
for $(T_t', X_t')_{t=1}^n$ i.i.d.\ replicates of $(T,X)$ independent of $\hat{\gamma}$, then
\begin{equation}\label{eqn:secondmoment}
\textstyle
|\hat{\mu}_{2}-\E[\eta^2]|=O_P\left(\frac{1}{n^{\frac{1}{3}}}\right)
\qtext{and}
|\hat{\mu}_{3}-\E[\eta^3]|=O_P\left(\frac{1}{\sqrt{n}}\right).
\end{equation}
As a result, 
\[
n^{\frac{1}{2}}|\hat{\mu}_{2}-\E[\eta^2]|^3 \toprob0
\qtext{and} 
n^{\frac{1}{2}}|\hat{\mu}_{3}-\E[\eta^3]|^3 \toprob0.
\]
\end{lemma}

\begin{proof}
 We begin with the third moment estimation. 
 For a new datapoint $(T,X)$ independent of $\hat{\gamma}$, define $\delta \defeq \ldot{X}{\gamma_0-\hat{\gamma}}$ so that $T - \inner{X}{\hat{\gamma}} = \delta + \eta$.
Since $\eta$ is independent of $(X, \hat{\gamma})$, and $\E[\eta]=0$, we have
\begin{align*}
\E_{X, \eta}[\left(\delta+\eta\right)^3]-3\E_{X, \eta}[\left(\delta+\eta\right)]\E_{X, \eta}[\left(\delta+\eta\right)^2]=\E[\eta^3]+\E_{X}[\delta^3]-3\E_{X}[\delta^2]\E_{X}[\delta]
\end{align*}
or equivalently
\begin{align}\label{eq:ident3}
\E[\eta^3]=\E_{X, \eta}[\left(\delta+\eta\right)^3]-3\E_{X, \eta}[\left(\delta+\eta\right)]\E_{X, \eta}[\left( \delta+\eta\right)^2]-\E_X[\delta^3]+3\E_X[\delta^2]\E_X[\delta].
\end{align}

Since $\E_X[|\delta|^3] \leq \E_X[\delta^6]^{\frac{1}{2}}=O_P(1/\sqrt{n})$ by Cauchy-Schwarz and our assumption on $\hat{\gamma}$, and $|\E_X[\delta]\E_X[\delta^2]| \leq \E_X[|\delta|^3]$ by Holder's inequality, the equality (\ref{eq:ident3}) implies that 
\begin{align*}
|\E[\eta^3]-(\E_{X,\eta}[\left(\delta+\eta\right)^3]-3\E_{X,\eta}[\delta+\eta]\E_{X,\eta}[\left( \delta+\eta\right)^2])| =O_P(1/\sqrt{n}).
\end{align*} 
Since $\E_{X,\eta}[\left(\delta+\eta\right)^6] =O(1)$, the central limit theorem, the strong law of large numbers, and Slutsky's theorem imply that
$$
\hat{\mu}_{3} - (\E_{X,\eta}[\left(\delta+\eta\right)^3]-3\E_{X,\eta}[\delta+\eta]\E_{X,\eta}[\left( \delta+\eta\right)^2]) = O_P(1/\sqrt{n}).
$$
Therefore, $$|\hat{\mu}_{3}-\E[\eta^3]|= O_P(1/\sqrt{n}).$$ 

The second moment estimation follows similarly using the identity,
$\E[\eta^2]=\E_{X,\eta}[\left(\delta+\eta\right)^2]-\E_X[\delta^2],$
and the fact that  $\E_X[\delta^2] \leq \E_X[|\delta|^3]^{\frac{2}{3}} =O_P(n^{-\frac{1}{3}})$ by Holder's inequality.
\end{proof}

In light of \lemref{proofstages} it suffices to estimate the vectors $q_0$ and $\gamma_0$ using $n$ sample points so that \begin{itemize}
\item $n^{\frac{1}{2}}\E_X[|\ldot{X}{\hat{q}-q_0}|^4]^{\frac{1}{4}}n^{-\frac{1}{2}} \toprob0 \Leftrightarrow \E_X[|\ldot{X}{\hat{q}-q_0}|^4]^{\frac{1}{4}}\toprob0$, 
\item $n^{\frac{1}{2}}\E_X[|\ldot{X}{\hat{\gamma}-\gamma_0}|^4]^{\frac{1}{4}}n^{-\frac{1}{2}} \toprob0 \Leftrightarrow \E_X[|\ldot{X}{\hat{\gamma}-\gamma_0}|^4]^{\frac{1}{4}}\toprob0$,
\item $n^{\frac{1}{2}}\E_X[|\ldot{X}{\hat{q}-q_0}|^6]^{\frac{1}{2}} \toprob0$, and
\item $n^{\frac{1}{2}}\E_X[|\ldot{X}{\hat{\gamma}-\gamma_0}|^6]^{\frac{1}{2}} \toprob0$,
\end{itemize}
and the rest of the conditions will follow.
To achieve these conclusions we use the following result on the performance of the Lasso.  The following theorem is distilled from \citep[Chapter 11]{HastieTiWa2015}.

{\begin{theorem}\label{LASSO}
Let $p,s \in \mathbb{N}$ with $s \leq p$ and $s=o(n^{2/3}/\log p )$ and $\sigma>0$, and suppose that we observe i.i.d.\ datapoints $(\tilde{Y}_i,\tilde{X}_i)_{i=1}^n$ distributed according to the model $\tilde{Y}=\ldot{\tilde{X}}{\beta_0}+w$ 
for an $s$-sparse $\beta_0\in \mathbb{R}^p$, $\tilde{X} \in \mathbb{R}^p$ with standard Gaussian  entries, and $w\in\R^p$ independent mean-zero noise with $\|w\|_{\infty} \leq \sigma$. 
Suppose that $p$ grows to infinity with $n$. Then with a choice of tuning parameter $\lambda_n = 2\sigma\sqrt{3\log p/n}$, the Lasso estimate $\hat{\beta}_0$ fit to this dataset satisfies $\|\hat{\beta}_0-\beta_0\|_2 = O_P(\sqrt{s \log p/n})$.
\end{theorem}
\begin{proof}

Using Theorem 11.1 and Example 11.2 of \cite{HastieTiWa2015}, we know that since $\tilde{X}$ has iid $N(0,1)$ entries, if  $\lambda_n = 2\sigma \sqrt{3 \log(p)/n}$, we have
\begin{equation}
\Pr\left[ \frac{\|\hat{\beta}_0-\beta_0\|_2}{\sigma\sqrt{3 s \log p/n}}> 1\right] \leq 2\exp\left\{-\frac{1}{2}\log(p)\right\} .
\end{equation}
Since $p$ grows unboundedly with $n$, for any fixed $\epsilon > 0$, we have that for some some finite $N_\epsilon$, if $n> N_\epsilon$, the right hand side is at most $\epsilon$. Thus we can conclude that: $\|\hat{\beta}_0-\beta_0\|_2 = O_{p}\left(\sqrt{s \log p/n}\right)$.
\end{proof}}

Notice that for $q_0$ we know 
\begin{align*}
Y&=\theta_0T+\ldot{X}{\beta_0}+\epsilon\\
&=\theta_0\ldot{X}{\gamma_0}+\theta_0\eta+\ldot{X}{\beta_0}+\epsilon \tag{from the definition of $T$}\\
&=\ldot{X}{q_0}+\theta_0\eta+\epsilon \tag{since $q_0=\theta_0 \gamma_0+\beta_0$}
\end{align*} Hence, 
\begin{equation*}
Y=\ldot{X}{q_0}+\epsilon+\theta_0\eta,
\end{equation*} and we know that the noise term, $\epsilon+\theta_0\eta$ is almost surely bounded by $C+CM=C(M+1)$. 
Hence, by Theorem \ref{LASSO}, our Lasso estimate $\hat{q}$ satisfies $\|\hat{q}-q_0\|_2 = O_P(\sqrt{{s \log p}{/n}})$. 
Similarly, our Lasso estimate $\hat{\gamma}$ satisfies $\|\hat{\gamma}-\gamma_0\|_2 = O_P(\sqrt{{s \log p}{/n}})$.

Now, since $X$ has iid mean-zero standard Gaussian components, we know that for all vectors $v \in \mathbb{R}^p$ and $a \in \mathbb{N}$ it holds $\E[|\ldot{X}{v}|^a]=O\left(\sqrt{a}^a\|v\|_2^a\right)$. Applying this to $v=\hat{q}-q$ and $v=\hat{\gamma}-\gamma_0$ for $a \in \{4,6\}$ we have
{\balignst
&\E_X[|\ldot{X}{\hat{q}-q_0}|^4]= O(\|\hat{q}-q_0\|_2^4) = O_{P}\left(\left[\sqrt{\frac{s \log p}{n}} \right]^4\right)\\
&\E_X[|\ldot{X}{\hat{\gamma}-\gamma_0}|^4]= O(\|\hat{\gamma}-\gamma_0\|_2^4) = O_{P}\left(\left[\sqrt{\frac{s \log p}{n}} \right]^4\right)\\
&\E_X[|\ldot{X}{\hat{q}-q_0}|^6]= O(\|\hat{q}-q_0\|_2^6)= O_{P}\left(\left[\sqrt{\frac{s \log p}{n}} \right]^6\right) \\
&\E_X[|\ldot{X}{\hat{\gamma}-\gamma_0}|^6]= O(\|\hat{\gamma}-\gamma_0\|_2^6)=O_{P}\left(\left[\sqrt{\frac{s \log p}{n}} \right]^6\right).
\ealignst
Now for the sparsity level $s=o\left(\frac{n^{2/3}}{\log p}\right)$ we have $\sqrt{\frac{s \log p}{n}}=o(n^{-\frac{1}{6}})$ which implies all of the desired conditions for \assumpref{first-stage}.}

\subsection{Checking Consistency of First Stage (\assumpref{first-stage-consistency})}
Next we prove that \assumpref{first-stage-consistency} is satisfied. Since $ \max_{\alpha \in S}\|\alpha\|_1 =2$ it suffices by Lemma \ref{thm:Sorth} to show that for our choices of $\hat{\gamma},\hat{q},\hat{\mu}_2$, and $\hat{\mu}_3$ we have
\begin{align}
&\E_X[|\ldot{X}{\hat{q}-q_0}|^8]^{\frac{1}{8}} \toprob0 \label{eqn:lasso_consistency1}\\
&\E_X[|\ldot{X}{\hat{\gamma}-\gamma_0}|^8]^{\frac{1}{8}} \toprob0 \label{eqn:lasso_consistency2}\\
&|\hat{\mu}_{2}-\E[\eta^2]| \toprob0 \label{eqn:lasso_consistency3}\\ 
&|\hat{\mu}_{3}-\E[\eta^3]| \toprob0 \label{eqn:lasso_consistency4}.
\end{align} Parts \eqnref{lasso_consistency3} and \eqnref{lasso_consistency4} follow directly from Lemma \ref{lem:proofstages}. 
Since $X$ consists of standard Gaussian entries, an analogous argument to that above implies that
\balignst
&\E_X[|\ldot{X}{\hat{q}-q_0}|^8]^{\frac{1}{8}}= O(\|\hat{q}-q_0\|_2) = O_{P}\left(\left[\sqrt{\frac{s \log p}{n}} \right]\right)\\
&\E_X[|\ldot{X}{\hat{\gamma}-\gamma_0}|^8]^{\frac{1}{8}}= O(\|\hat{\gamma}-\gamma_0\|_2) = O_{P}\left(\left[\sqrt{\frac{s \log p}{n}} \right]\right).
\ealignst
Now for the sparsity level $s=o(\frac{n^{\frac{2}{3}}}{\left(M+1\right)^2\log p})$ we have $\sqrt{\frac{s \log p}{n}}=o(1)$ which implies also conditions \eqnref{lasso_consistency1} and \eqnref{lasso_consistency2}.

\subsection{Checking Convexity Conditions (\assumpargref{consistency}{convexity})}
Finally, we establish the convexity conditions (\assumpargref{consistency}{convexity}).  We consider $\Theta=\mathbb{R}$, which is convex.
Without loss of generality, assume $3 \E[\eta^2]^2> \E[\eta^{4}]$; otherwise, one can establish the convexity conditions for $-m$.
Let $F_n(\theta) = \frac{1}{n} \sum_{t=1}^n m(Z_t, \theta, \hat{h}(X_t))$.
Since $F_n$ is continuously differentiable, $F_n$ is the derivative of a convex function whenever $\grad F_n(\theta) \geq 0$, for all $\theta \in \Theta$.
Since $F_n$ is linear in $\theta$ we have for all $\theta \in \Theta$
\[
\grad F_n(\theta) = \frac{1}{n} \sum_{t=1}^n - (T_t - \inner{\hat{\gamma}}{X_t})^{4} + (T_t - \inner{\hat{\gamma}}{X_t}) \hat{\mu}_3 + 3  (T_t -\inner{\hat{\gamma}}{X_t})^2 \hat{\mu}_2,
\]
the established consistency of $(\hat\gamma, \hat{\mu}_3, \hat{\mu}_2)$ and Slutsky's theorem imply that
\[
\grad F_n(\theta) - \frac{1}{n} \sum_{t=1}^n - (T_t - \inner{\gamma}{X_t})^{4} + (T_t - \inner{\gamma}{X_t}) \E[\eta^3] + 3  (T_t -\inner{{\gamma}}{X_t})^2 \E[\eta^2] \toprob 0.
\]
The strong law of large numbers now yields
\[
\grad F_n(\theta) - (3\E[\eta^2]^2- \E[\eta^4] ) \toprob 0.\]
Hence, 
\[\Pr(\grad F_n(\theta) < 0) \leq \Pr(|\grad F_n(\theta) - (3\E[\eta^2]^2- \E[\eta^4] )| > 3\E[\eta^2]^2- \E[\eta^4] ) \to 0,
\] verifying \assumpargref{consistency}{convexity}.
The proof is complete. 

\section{Proofs of Auxiliary Lemmata}\label{sec:lemmata}

\subsection{Proof of \lemref{conditionhighprob}}\label{sec:proof_conditionhighprob}
Since each $Y_n$ is binary, and $Y_n \toprob 1$, for every $\epsilon>0$, 
$$\Pr[|X_n (1-Y_n)|>\epsilon] \leq \Pr[ Y_n = 0 ] = \Pr[ |1 - Y_n | > 1/2 ] \rightarrow 0.$$
Hence, $X_n (1-Y_n) \toprob 0$. Both advertised claims now follow by Slutsky's theorem \citep[Thm.\ 2.8]{vander}.

\subsection{Proof of \lemref{conditional-moment-convergence}}\label{sec:proof_conditional-moment-convergence}
Let $X_{n,i}$ denote the $i$-th coordinate of $X_n$, i.e. $\|X_n\|_p^p = \sum_{i=1}^d X_{n,i}^p$. By the assumption of the lemma, we have that for every $\epsilon, \delta$, there exists $n(\epsilon,\delta)$, such that for all $n\geq n(\epsilon,\delta)$:
\begin{equation*}
\Pr\left[ \max_{i}\E\left[ |X_{n,i}|^p | Z_n \right] > \epsilon \right] < \delta
\end{equation*}
Let ${\cal E}_n$ denote the event $\{\max_{i} \E\left[ |X_{n,i}|^p | Z_n \right] \leq  \epsilon\}$. Hence, $\Pr[{\cal E}_n]\geq 1-\delta$, for any $n\geq n(\epsilon,\delta)$. By Markov's inequality, for any $n\geq n\left(\nicefrac{\epsilon^{p}\delta}{2d}, \nicefrac{\delta}{2d}\right)$, the event ${\cal E}_n$ implies that:
\begin{equation*}
\Pr\left[ |X_{n,i}|^{p} > \epsilon^{p} | Z_n\right] \leq \frac{\E\left[ |X_{n,i}|^{p} | Z_n\right]}{\epsilon^{p}} \leq \frac{\delta}{2d}
\end{equation*}
Thus, we have:
\begin{align*}
\Pr[|X_{n,i}| > \epsilon ] =~& \E\left[\Pr\left[|X_{n,i}|^{p} > \epsilon^{p} | Z_n\right]\right]\\
 =~& \E\left[\Pr[|X_{n,i}|^{p} > \epsilon^{p} | Z_n] | {\cal E}_n\right]\cdot \Pr[{\cal E}_n] +\E\left[\Pr[|X_{n,i}|^{p} > \epsilon^{p} | Z_n] | \neg {\cal E}_n\right]\cdot \Pr[\neg {\cal E}_n]
\leq \frac{\delta}{d}
\end{align*}
By a union bound over $i$, we have that $\Pr[ \max_{i} |X_{n,i}| > \epsilon] \leq \delta$. Hence, we also have that for any $\epsilon, \delta$, for any $n\geq n(\nicefrac{\epsilon^p \delta}{2d}, \nicefrac{\delta}{2d})$, $\Pr[\|X_n\|_{\infty} > \epsilon]\leq \delta$, which implies $X_n\toprob 0$.
}

\end{document}